\documentclass[10pt,journal,letterpaper,compsoc]{IEEEtran}

\ifCLASSINFOpdf
\usepackage[pdftex]{graphicx}
\DeclareGraphicsExtensions{.jpg,.png,.pdf}
\usepackage{pdflscape}
\else
\usepackage[dvips]{graphicx}
\DeclareGraphicsExtensions{.eps}
\usepackage{lscape}
\fi

\usepackage{afterpage}

\usepackage[sort,numbers]{natbib}
\renewcommand{\citename}{\citet}
\renewcommand{\cite}{\citep}
\usepackage{natbibspacing}

\usepackage[cmex10]{amsmath}
\usepackage{amssymb,amsthm,latexsym}

\usepackage{algorithm}
\usepackage{algorithmic}

\usepackage{hyperref}

\usepackage[font=small,labelfont=bf]{caption}
\usepackage[font=footnotesize]{subfig}

\usepackage{tabularx} 

\usepackage{pgfplots}
\pgfplotsset{compat=1.5, every axis/.append style={font=\small, /pgf/number format/1000 sep={}}, every mark/.append style={solid},}
\usepackage{tikz}
\usetikzlibrary{quotes, arrows.meta, angles, calc, 3d, shapes, intersections, plotmarks, positioning} 
\usepgfplotslibrary{statistics, polar, groupplots}
\usepackage{tikzscale}

\newcommand{\eg}{e.g.,~}

\usepackage{sg-macros}

\graphicspath{{./}}

\begin{document}
\title{Deep Declarative Networks: A New Hope}

\iffalse
\author{Stephen~Gould,~\IEEEmembership{Member,~IEEE,}
		Richard~Hartley,~\IEEEmembership{Fellow,~IEEE,}
        Dylan~Campbell,~\IEEEmembership{Member,~IEEE}%
\IEEEcompsocitemizethanks{\IEEEcompsocthanksitem S. Gould and D. Campbell are
	with the Australian Centre for Robotic Vision and the Research School of
	Computer Science at the Australian National University. R. Hartley is with
	the the Australian Centre for Robotic Vision and the
	Research School of Engineering at the Australian National University.
	E-mail: stephen.gould@anu.edu.au}%
}

\markboth{PAMI}{Gould \etal: Deep Declarative Networks}
\else
\author{Stephen~Gould,
	Richard~Hartley,
	Dylan~Campbell%
	\IEEEcompsocitemizethanks{\IEEEcompsocthanksitem S. Gould and D. Campbell are
		with the Australian Centre for Robotic Vision and the Research School of
		Computer Science at the Australian National University. R. Hartley is with
		the the Australian Centre for Robotic Vision and the
		Research School of Engineering at the Australian National University.
		E-mail: stephen.gould@anu.edu.au}%
}

\markboth{preprint}{Gould \etal: Deep Declarative Networks}
\fi

\IEEEtitleabstractindextext{%
\begin{abstract}
We explore a new class of end-to-end learnable models wherein data processing
nodes (or network layers) are defined in terms of desired behavior rather than an
explicit forward function. Specifically, the forward function is implicitly
defined as the solution to a mathematical optimization problem. Consistent with nomenclature in
the programming languages community, we name these models \emph{deep declarative networks}.
Importantly, we show that the class of deep declarative networks subsumes
current deep learning models. Moreover, invoking the implicit function theorem, we
show how gradients can be back-propagated through many declaratively defined data processing
nodes thereby enabling end-to-end learning. We show how these declarative processing
nodes can be implemented in the popular PyTorch deep learning software library allowing
declarative and imperative nodes to co-exist within the same network. We also provide
numerous insights and illustrative examples of declarative nodes and demonstrate their
application for image and point cloud classification tasks.
\end{abstract}

}

\maketitle

\IEEEdisplaynontitleabstractindextext
\IEEEpeerreviewmaketitle

\IEEEraisesectionheading{\section{Introduction}\label{sec:introduction}}

\IEEEPARstart{M}{odern} deep learning models are composed of
parametrized processing nodes (or layers) organized in a directed
graph. There is an entire zoo of different model architectures
categorized primarily by graph structure and mechanisms for
parameter sharing~\cite{Lecun:Nature15}. In all cases the function for
transforming data from the input to the output of a processing node is
explicitly defined. End-to-end learning is then achieved by
back-propagating an error signal along edges in the graph and
adjusting parameters so as to minimize the error. Almost universally,
the error signal is encoded as the gradient of some global objective
(or regularized loss) function and stochastic gradient descent based methods are used to
iteratively update parameters. Automatic differentiation (or
hand-derived gradients) is used to compute the derivative of the
output of a processing node with respect to its input, which is then
combined with the chain rule of differentiation proceeding backwards
through the graph. The error gradient can thus be calculated
efficiently for all parameters of the model as the signal is passed
backwards through each processing node.

In this paper we advocate for a new class of end-to-end learnable
models, which we call \emph{deep declarative networks} (DDNs), and
which collects several recent works on differentiable optimization~\cite{Gould:TR2016, Amos:ICML2017, Amos:PhD2019, Agrawal:TR2019, Agrawal:NIPS2019}
and related ideas~\cite{Chen:NIPS2018, Tschiatschek:NIPS2018, Wang:ICML2019} into a single framework. Instead
of explicitly defining the forward processing function, nodes in a DDN
are defined implicitly by specifying behavior. That is, the
input--output relationship for a node is defined in terms of an objective
and constraints in a mathematical optimization problem, the output being the
solution of the problem conditioned on the input (and
parameters). Importantly, as we will show, we can still perform
back-propagation through a DDN by making use of implicit
differentiation. Moreover, the gradient calculation does not require
knowledge of the method used to solve the optimization problem, only
the form of the objective and constraints, thereby allowing any state-of-the-art
solver to be used during the forward pass.

DDNs subsume conventional deep learning models in that any explicitly
defined forward processing function can also be defined as a node in a
DDN. Furthermore, declaratively defined nodes and explicitly defined
nodes can coexist within the same end-to-end learnable model. To make
the distinction clear, when both types of nodes appear in the same
model we refer to the former as \emph{declarative nodes} and the
latter as \emph{imperative nodes}.  To this end, we have developed a reference DDN
implementation within PyTorch~\cite{PyTorch},
a popular software library for deep learning,
supporting both declarative and imperative nodes.

We present some theoretical results that show the conditions under
which exact gradients can be computed and the form of such gradients. We
also discuss scenarios in which the exact gradient cannot be computed
(such as non-smooth objectives) but a descent direction can still be
found, allowing stochastic optimization of model parameters to
proceed. These ideas are explored through a series of illustrative
examples and tested experimentally on the problems of image and point
cloud classification using modified ResNet~\cite{He:CVPR2016} and
PointNet~\cite{Qi:CVPR2017} architectures, respectively.

The decisive advantage of the declarative view of deep neural networks is
that it enables the use of classic constrained and unconstrained optimization algorithms
as a modular component within a larger, end-to-end learnable network.
This extends the concept of a neural network layer to include, for example,
geometric model fitting, such as relative or absolute pose solvers or bundle adjustment,
model-predictive control algorithms, expectation-maximization, matching, optimal transport,
and structured prediction solvers to name a few.
Moreover, the change in perspective can help us envisage variations of
standard neural network operations with more desirable properties,
such as robust feature pooling instead of standard average pooling.
There is also the potential to reduce opacity and redundancy in networks by
incorporating local model fitting as a component within larger models.
For example, we can directly use the (often nonlinear) underlying physical
and mathematical models rather than having to re-learn these within the
network. Importantly, this allows us to provide guarantees and enforce
hard constraints on representations within a model (for example, that a
rotation within a geometric model is valid or that a permutation
matrix is normalized correctly).
Furthermore, the approach is still applicable when no closed form solution exists,
allowing sophisticated approaches with non-differentiable steps (such as
RANSAC~\cite{Fischler:1981}) to be used internally.
Global end-to-end learning of network parameters is still possible even in this case.

Since this is a new approach, there are still several challenges to be
addressed. We discuss some of these in this paper but
leave many for future work by us and the community.
Some preliminary ideas relating to declarative networks appear in
previous works (discussed below) but, to the best of our knowledge,
this paper is the first to present a general coherent framework for
describing these models.

\section{Background and Related Work}
\label{sec:background}

The ability to differentiate through declarative nodes relies on the
implicit function theorem, which has a long history whose roots can be
traced back to works by Descartes, Leibniz,
Bernoulli and Euler~\cite{Scarpello:2002}. It was Cauchy who
first placed the theorem on rigorous mathematical grounds
and Dini who first presented
the theorem in its modern multi-variate form~\cite{Krantz:2013, Dontchev:2014}.
Roughly speaking, the theorem
states conditions under which the derivative of a variable $y$ with
respect to another variable $x$ exists for implicitly defined functions,
$f(x, y) = 0$, and provides a means for computing the
derivative of $y$ with respect to $x$ when it does exist.

In the context of deep learning, it is the derivative of the output of
a (declarative) node with respect to its input that facilitates
end-to-end parameter learning by
back-propagation~\cite{Rumelhart:Nature86, Lecun:Nature15}. In this
sense the learning problem is formulated as an optimization problem on
a given error metric or regularized loss function. When declarative nodes appear in
the network, computing the network output and hence the loss function
itself requires solving an inner optimization problem. Formally, we
can think of the learning problem as an upper optimization problem and
the network output as being obtained from a lower optimization problem within a bi-level
optimization framework~\cite{Stackelberg:2011, Bard:1998}.

Bi-level optimization (and the need for implicit differentiation) has
appeared in various settings in the machine learning literature, most
notably for the problem of meta-learning.
For example, \citename{Do:NIPS07} consider the problem of determining
regularization parameters for log-linear models formulated as a
bi-level optimization problem. \citename{Domke:AISTATS12} addresses
the problem of
learning parameters of continuous energy-based models wherein inference
(equivalent to the forward pass in a neural network) requires finding
the minimizer of a so-called energy function. The resulting learning problem
is then necessarily bi-level. Last, \citename{Klatzer:CVWW2015} propose
a bi-level optimization framework for choosing hyper-parameter settings
for support vector machines that avoids the need for cross-validation.

In computer vision and image processing, bi-level optimization has
been used to formulate solutions to pixel-labeling problems.
\citename{Samuel:CVPR09} propose learning parameters of a continuous
Markov random field with bi-level optimization and apply their technique
to image denoising and in-painting. The approach is a special case of
energy-based model learning~\cite{Domke:AISTATS12}. \citename{Ochs:2015}
extend bi-level optimization to handle non-smooth
lower-level problems and apply their method to image segmentation tasks.

Recent works have started to consider the question of placing specific
optimization problems within deep learning models~\cite{Gould:TR2016,
Amos:ICML2017, Chen:NIPS2018, Tschiatschek:NIPS2018, Wang:ICML2019}.
These approaches can be thought of as
establishing the groundwork for DDNs by developing specific declarative
components. \citename{Gould:TR2016} summarize some general results for
differentiating through unconstrained, linearly constrained and
inequality constrained optimization problems. In the case of
unconstrained and equality constrained problems the results are exact
whereas for inequality constrained problems they make use of an
interior-point approximation. We extend these results to the case of
exact differentiation for problems with nonlinear equality and
inequality constraints.

\citename{Amos:ICML2017} also show how to differentiate through an
optimization problem, for the specific case of quadratic
programs (QPs). A full account of the work, including discussion of
more general cone programs, appears in \citename{Amos:PhD2019}.
Along the same lines \citename{Agrawal:TR2019} report results for
efficiently differentiating through cone programs with millions of parameters.
In both cases (quadratic programs and cone programs) the problems are convex and
efficient algorithms exist for finding the minimum. We make no restriction on
convexity for declarative nodes but still assume that an
algorithm exists for evaluating them in the forward pass. Moreover,
our paper establishes a unified framework for viewing these works
in an end-to-end learnable model---the deep declarative network (DDN).

Other works consider the problem of differentiating through discrete
submodular problems~\cite{Djolonga:NIPS2017, Tschiatschek:NIPS2018}. These
problems have the nice property that minimizing a convex relaxation still
results in the optimal solution to the submodular minimization problem,
which allows the derivative to be computed~\cite{Djolonga:NIPS2017}.
For submodular maximization there exist polynomial time algorithms
to find approximate solutions. Smoothing of these solutions results
in a differentiable model as demonstrated in~\citename{Tschiatschek:NIPS2018}.

Close in spirit to the idea of a deep declarative network is the
recently proposed SATNet~\cite{Wang:ICML2019}. Here MAXSAT problems are approximated by
solving a semi-definite program (SDP), which is differentiable. A fast solver allows
the problem to be evaluated efficiently in the forward pass and thanks to implicit
differentiation there is no need to explicitly unroll the optimization procedure,
and hence store Jacobians, making it memory and time efficient in the backward pass.
The method is applied to solving Sudoku problems presented as images, which requires
that the learned network encode logical constraints.

Another interesting class of models that fit into the deep declarative networks
framework are deep equilibrium models recently proposed by \citename{Bai:NIPS2019}. Here the model executes a sequence of fixed-point iterations $y^{(t)} = f(x, y^{(t-1)})$ until convergence in the forward pass. \citename{Bai:NIPS2019} show that rather than back-propagating through the unbounded sequence of fixed-point iterations, the derivative of the solution $y$ with respect to input $x$ can be computed directly via the implicit function theorem by observing that $y$ satisfies implicit function $f(x, y) - y = 0$.

\citename{Chen:NIPS2018} show how to differentiate through ordinary differential equation (ODE) initial value problems at some time $T$ using the adjoint sensitivity method, 
and view residual networks \cite{He:CVPR2016} as a discretization of such problems.
This approach can be interpreted as solving a feasibility problem with an
integral constraint function, and hence an exotic type of declarative node. The
requisite gradients in the backward pass are computed elegantly by solving a second
augmented ODE. While such problems can be included within our declarative framework,
in this paper we focus on results for twice-differentiable objective and constraints
functions that can be expressed in closed form.

There have also been several works that have proposed differentiable
components based on optimization problems for addressing specific
tasks. In the context of video classification,
\citename{Fernando:ICML2016, Fernando:IJCV2017} show how to
differentiate through a rank-pooling operator~\cite{Fernando:CVPR2015} within
a deep learning model, which
involves solving a support vector regression problem. The approach was
subsequently generalized to allow for a subspace representation of the video
and learning on a manifold~\cite{Cherian:CVPR2017}.

\citename{SantaCruz:PAMI2018} propose a deep learning model for
learning permutation matrices for visual attribute ranking and
comparison. Two variants are proposed both relaxing the permutation
matrix representation to a doubly-stochastic matrix
representation. The first variant involves iteratively normalizing
rows and columns to approximately project a positive matrix onto the
set of doubly-stochastic matrices. The second variant involves
formulating the projection as a quadratic problem and solving
exactly.

\citename{Lee:CVPR2019} consider the problem of few-shot learning
for visual recognition. They embed a differentiable QP~\cite{Amos:ICML2017} 
into a deep learning model that allows linear classifiers to be
trained as a basis for generalizing to new visual categories. Promising
results are achieved on standard benchmarks and they report low training
overhead in that solving the QP takes about the same time as
image feature extraction.

In the context of planning and control, \citename{Amos:NIPS2018} propose a
differentiable model predictive controller in a reinforcement learning setting.
They are able to demonstrate superior performance on classic pendulum and cartpole
problems.
De Avila Belbute-Peres et al.~\cite{Peres:NIPS2018}
show how to differentiate through physical
models, specifically, the optimal solution of a linear complementarity
problem, which allows a simulated physics environment to be placed within an
end-to-end learnable system.

Last, imposing hard constraints on the output of a deep neural network using
a Krylov subspace method was investigated in~\citename{Neila:TR2017}.
The approach is applied to the task of human pose estimation
and demonstrates the feasibility of training a very high-dimensional model that
enforces hard constraints, albeit without improving over a softly
constrained baseline.

\section{Notation}
\label{sec:notation}

Our results require differentiating vector-valued functions with
respect to vector arguments. To assist presentation we clarify
our notation here.
Consider the function $f(x, y, z)$ where both $y$ and $z$ are
themselves functions of $x$. We have
\begin{align}
\ddx{} f
&= \frac{\partial f}{\partial x} \ddx{x} +
\frac{\partial f}{\partial y} \ddx{y} +
\frac{\partial f}{\partial z} \ddx{z}
\label{eqn:scalar_notation}
\end{align}
by the chain rule of differentiation. For functions taking vector arguments, $f :
\reals^n \to \reals$, we write the derivative vector as
\begin{align}
\dd f &= \left[ \frac{\partial f}{\partial x_1}, \frac{\partial f}{\partial x_2}, \ldots, \frac{\partial f}{\partial x_n} \right] \in \reals^{1 \times n}.
\end{align}
For vector-valued functions $f: \reals \to \reals^m$ we write
\begin{align}
\dd f &= \left[ \ddx{f_1}, \ldots, \ddx{f_m} \right]\transpose \in \reals^{m \times 1}.
\end{align}

More generally, we define the derivative $\dd f$ of $f : \reals^n
\to \reals^m$ as an $m \times n$ matrix with entries
\begin{align}
\Big(\dd f(x)\Big)_{ij} &= \frac{\partial f_i}{\partial x_j}(x).
\end{align}
Then the chain rule for $h(x) = g(f(x))$ is simply
\begin{align}
\dd h(x) &= \dd g(f(x)) \dd f(x)
\end{align}
where the matrices automatically have the right dimensions and
standard matrix-vector multiplication applies. When taking partial
derivatives we use a subscript to denote the formal variable over
which the derivative is computed (the remaining variables fixed), for example
$\dd[X] f(x, y)$. For brevity we use the shorthand
$\dd[XY]^2 f$ to mean $\dd[X] (\dd[Y] f)\transpose$.

When no subscript is given for multi-variate functions we take $\dd$
to mean the total derivative with respect to the
independent variables. So, with $x$ as the independent variable,
the vector version of \eqnref{eqn:scalar_notation} becomes
\begin{align}
  \dd f &= \dd[X] f + \dd[Y] f \, \dd y + \dd[Z] f \, \dd z.
\end{align}

\section{Deep Declarative Networks}
\label{sec:ddn}

\begin{figure}
  \centering
  {\large \includegraphics[]{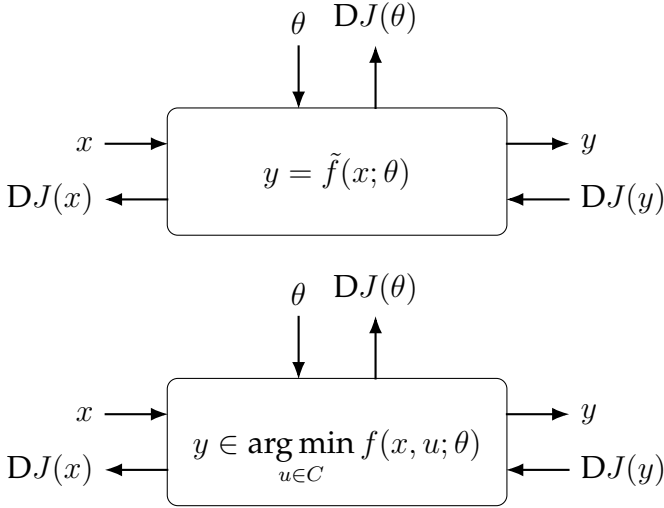}} 
  \caption{Parametrized data processing nodes in an end-to-end
    learnable model with global objective function $J$.
    During the forward evaluation pass of an
    imperative node (top) the input $x$ is transformed into output $y$ based
    on some explicit parametrized function $\tilde{f}(\cdot;
    \theta)$. During the forward evaluation pass of a declarative node (bottom)
    the output $y$ is computed as the minimizer of some parametrized
    objective function $f(x, \cdot; \theta)$. During the backward
    parameter update pass for either node type, the gradient of
    the global objective function with respect
    to the output $\dd J(y)$ is propagated backwards via the chain
    rule to produce gradients with respect to the input $\dd J(x)$ and
    parameters $\dd J(\theta)$.
  }
  \label{fig:ddn_nodes}
\end{figure}

We are concerned with data processing nodes in deep neural networks
whose output $y \in \reals^m$ is defined as the solution to a
constrained optimization problem parametrized by the input $x \in
\reals^n$. In the most general setting we have
\begin{align}
  y \in \argmin_{u \in C} f(x, u)
  \label{eqn:ddn_general}
\end{align}
where $f : \reals^n \times \reals^m \to \reals$ is an arbitrary
parametrized objective function and $C \subseteq \reals^m$ is an
arbitrary constraint set (that may also be parametrized by $x$).%
\footnote{In practical deep learning systems inputs and outputs are
  generally one or more multi-dimensional tensors. For clarity of exposition we consider
  inputs and outputs as single vectors without loss of generality.}

To distinguish such processing nodes from processing nodes (or layers)
found in conventional deep learning models, which specify an explicit
mapping from input to output, we refer to the former as declarative
nodes and the latter as imperative nodes, which we formally define as
follows.

\begin{definition}
  {\bf (Imperative Node).} An \emph{imperative} node is one
  in which the implementation of the forward processing function
  $\tilde{f}$ is explicitly defined. The output is then defined
  mathematically as
  \begin{align*}
    y &= \tilde{f}(x; \theta)
  \end{align*}
  where $\theta$ are the parameters of the node (possibly empty).
  \label{def:imperative_node}
\end{definition}

\begin{definition}
  {\bf (Declarative Node).} A \emph{declarative} node is one
  in which the exact implementation of the forward processing function
  is not defined; rather the input--output relationship ($x \mapsto y$)
  is defined in terms of behavior specified as the solution to an
  optimization problem
  \begin{align*}
    y \in \argmin_{u \in C} f(x, u; \theta)
  \end{align*}
  where $f$ is the (parametrized) objective function, $\theta$ are
  the node parameters (possibly empty), and $C$ is the set of feasible
  solutions, that is, constraints.
  \label{def:declarative_node}
\end{definition}

\figref{fig:ddn_nodes} is a schematic illustration of the
different types of nodes. Since nodes form part of an end-to-end
learnable model we are interested in propagating gradients backwards
through them. In the sequel we make no distinction between the
inputs and the parameters of a node as these are treated the same for the
purpose of computing gradients.

The definition of declarative nodes (Definition~\ref{def:declarative_node})
is very general and encompasses many subproblems that we would like embedded
within a network---robust fitting, projection onto a constraint
set, matching, optimal transport, physical simulations, etc.
Some declarative nodes may not be efficient to evaluate (that is, solve the
optimization problem specifying its behavior), and in some cases the gradient
may not exist (e.g., when the feasible set or output space is discrete).
Nevertheless, under certain conditions on the objective and constraints, covering
a wide range of problems or problem approximations (such as shown by
\citename{Wang:ICML2019} for MAXSAT), the declarative node is differentiable and can
be placed within an end-to-end learnable model as we discuss below.

\subsection{Learning}
\label{sec:learning}

We make no assumption about \emph{how} the optimal solution $y$ is
obtained in a declarative node, only that there exists some algorithm
for computing it.  The consequence of this assumption is that, when
performing back-propagation, we do not need to propagate gradients
through a complicated algorithmic procedure from output to
input. Rather we rely on implicit differentiation to directly compute
the gradient of the node's output with respect to its input (or
parameters). Within the context of end-to-end learning with respect to
some global objective (i.e.,~loss function) this becomes a bi-level
(multi-level) optimization problem~\cite{Bard:1998, Dempe:2015},
which was first studied in the context of two-player leader--follower
games by Stackelberg in the 1930s~\cite{Stackelberg:2011}. Formally, we
can write
\begin{align}
  \begin{array}{ll}
    \text{minimize} & J(x, y) \\
    \text{subject to} & y \in \argmin_{u \in C} f(x, u)
  \end{array}
  \label{eqn:bilevel}
\end{align}
where the minimization is over all tunable parameters in the network.
Here $J(x, y)$ may depend on $y$ through additional layers of
processing. Note also that $y$ is itself a function of $x$. As such
minimizing the objective via gradient descent requires the following
computation
\begin{align}
  \dd J(x, y) &= \dd[X] J(x, y) + \dd[Y] J(x, y) \dd y(x),
\end{align}
The key challenge for declarative nodes then is the calculation of
$\dd y(x)$, which we will discuss in \secref{sec:backprop}. A
schematic illustration of a bi-level optimization problem and
associated gradient calculations is shown in \figref{fig:bilevel_opt}.

\begin{figure}
	\centering
	{\includegraphics[]{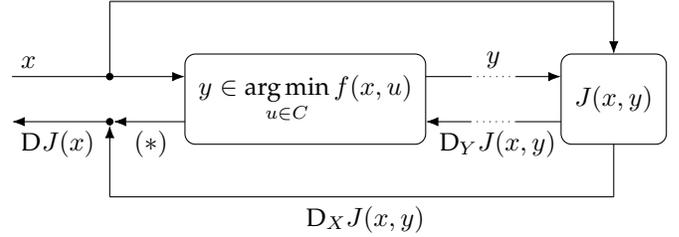}}
	\caption{Bi-level optimization problem showing back-propagation of gradients through
		a deep declarative node. The quantity $(*)$ is $\dd[Y]J(x, y) \dd y(x)$ which
		when added to $\dd[X] J(x, y)$ gives $\dd J(x, y(x))$. The bypass connections
		(topmost and bottommost paths) do not exist when the upper-level objective $J$
		only depends on $x$ through $y$. Moreover, if $f$ appears in $J$ as the only
		term involving $y$ then $\dd[Y] J(x, y)$ is zero and the backward edge $(*)$ is
		not required. That is, $\dd J(x) = \dd[X] J(x, y)$.}
	\label{fig:bilevel_opt}
\end{figure}

The higher-level objective $J$ is often defined as the sum of
various terms (including loss terms and regularization terms) and
decomposes over individual examples from a training dataset. In such cases
the gradient of $J$ with respect to model parameters also decomposes as the sum of gradients for losses on
each training example (and regularizers).

A simpler form of the gradient arises when the lower-level objective function $f$ appears in the upper-level objective $J$ as the only term involving $y$. Formally, if $J(x, y) = g(x, f(x, y))$ and
the lower-level problem is unconstrained (that is, $C = \reals^m$), then
\begin{align}
\begin{split}
\dd J(x, y) &= \dd[X\,] g(x, f) + \dd[F\,] g(x, f)\left( \dd f + \dd[Y\!] f \, \dd y\right) \\
&= \dd[X\,] g(x, f) + \dd[F\,] g(x, f) \dd f
\end{split}
\end{align}
since $\dd[Y] f(x, y) = 0$ by virtue of $y$ being an unconstrained minimum of $f(x, \cdot)$.
For example, if $J(x, y) = f(x, y)$, then $\dd J(x, y) = \dd[X] f(x, y)$.
This amounts to a variant of block coordinate descent on $x$ and $y$, where $y$ is
optimized fully during each step,
and in such cases we do not need to propagate gradients backwards through the
declarative node. The remainder of this paper is focused on the more general case
where $y$ appears in other terms in the loss function (possibly through composition with
other nodes) and hence calculation of $\dd y$ is required.

\subsection{Common sub-classes of declarative nodes}
\label{sec:common}

It will be useful to consider three very common cases of the general
setting presented in~\eqnref{eqn:ddn_general}, in increasing levels of
sophistication. The first is the unconstrained case,
\begin{align}
  y &\in \argmin_{u \in \reals^m} f(x, u)
  \label{eqn:dde_unconst}
\end{align}
where the objective function $f$ is minimized over the entire
$m$-dimensional output space. The second is when the feasible set is
defined by $p$ possibly nonlinear equality constraints,
\begin{align}
  \begin{array}{rllr}
    y \in& {\textstyle\argmin_{u \in \reals^m}} & f(x, u) \\
    & \text{subject to} & h_i(x, u) = 0, & i = 1, \ldots, p.
  \end{array}
  \label{eqn:dde_equ_const}
\end{align}
The third is when the feasible set is defined by $p$ equality and
$q$ inequality constraints,
\begin{align}
  \begin{array}{rlll}
    y \in& {\textstyle\argmin_{u \in \reals^m}} & f(x, u) \\
    & \text{subject to} & h_i(x, u) = 0, & i = 1, \ldots, p \\
    & & g_i(x, u) \leq 0, & i = 1, \ldots, q.
  \end{array}
\label{eqn:dde_inequ_const}
\end{align}

Note that in all cases we can think of $y$ as an implicit function of
$x$, and will often write $y(x)$ to make the input--output relationship clear.
We now show that, under mild conditions, the gradient of $y$ with
respect to $x$ can be computed for these and other special cases
without having to know anything about the algorithmic procedure used
to solve for $y$ in the first place.

In the sequel it will be useful to characterize solutions $y$ using the
concept of regularity due to \citename[\S 3.3.1]{Bertsekas:2004}.

\begin{definition}
	{\bf (Regular Point).} A feasible point $u$ is said to be \emph{regular} if
	the equality constraint gradients $\dd[U] h_i$ and the active inequality
	constraint gradients $\dd[U] g_i$ are linearly independent, or there are no
	equality constraints and the inequality constraints are all inactive at
	$u$.\footnote{An inequality constraint $g_i$ is active for a feasible point
		$u$ if $g_i(x, u) = 0$ and inactive if $g_i(x, u) < 0$.}
\end{definition}

\subsection{Back-propagation through declarative nodes}
\label{sec:backprop}

Our results for computing gradients are based on implicit
differentiation. We begin with the unconstrained case,
which has appeared in different guises in several previous works~\cite{Gould:TR2016}.
The result is a straightforward application of Dini's implicit function
theorem~\cite[p19]{Dontchev:2014} applied to the first-order optimality condition
$\dd[Y] f(x, y) = 0$.

\begin{proposition}[Unconstrained]
Consider a function $f: \reals^n \times \reals^m \to \reals$. Let
\begin{align*}
  y(x) &\in \argmin_{u \in \reals^m} f(x, u).
\end{align*}
Assume $y(x)$ exists and that $f$ is second-order differentiable
in the neighborhood of the point $(x, y(x))$.
Set $H = \dd[YY]^2 f(x, y(x)) \in \reals^{m \times m}$ and $B = \dd[XY]^2 f(x, y(x)) \in \reals^{m \times n}$.
Then for $H$ non-singular the derivative of $y$ with respect to $x$ is
\begin{align*}
\dd y(x) &= -H^{-1} B.
\end{align*}
\label{prop:unconstrained}
\end{proposition}
\vspace{-8mm}
\begin{proof}
For any optimal $y$, the first-order optimality condition requires
$\dd[Y] f(x, y) = \zeros_{1 \times m}$. The result then follows from
the implicit function theorem: transposing and differentiating both sides
with respect to $x$ we have
\begin{align}
\begin{split}
\zeros_{m \times n}
&= \dd \left( \dd[Y] f(x, y) \right)\transpose \\
&= \dd[XY]^2 f(x, y) + \dd[YY]^2 f(x, y) \dd y(x)
\end{split}
\end{align}
which can be rearranged to give
\begin{align}
\dd y(x) &= -\left(\dd[YY]^2 f(x, y)\right)^{-1} \dd[XY]^2 f(x, y)
\end{align}
when $\dd[YY]^2 f(x,y)$ is non-singular.
\end{proof}

In fact, the result above holds for any stationary point of $f(x, \cdot)$ not
just minima. However, for declarative nodes it is the minima that are of
interest. Our next result gives the gradient for equality constrained declarative
nodes, making use of the fact that the optimal solution of a constrained
optimization problem is a stationary point of the Lagrangian associated
with the problem.

\begin{proposition}[Equality Constrained]
  Consider functions $f: \reals^n \times \reals^m \to \reals$ and $h: \reals^n \times
  \reals^m \to \reals^p$. Let
  \begin{align*}
  	\arraycolsep=2pt
    \begin{array}{rllr}
      y(x) \in& {\textstyle\argmin_{u \in \reals^m}} & f(x, u) &\\
      & \text{subject to} & h_i(x, u) = 0, & i = 1, \ldots, p.
    \end{array}
  \end{align*}
  Assume that $y(x)$ exists, that $f$ and $h = [h_1,\dots,h_p]\transpose$ are
  second-order differentiable in the neighborhood of $(x, y(x))$,
  and that $\text{rank}(\dd[Y] h(x, y)) = p$.
  Then for $H$ non-singular
  \begin{align*}
    \dd y(x) &= H^{-1} \! A\transpose \!\! \left( A H^{-1} \! A\transpose \right)^{\!-1} \!\!\left(A H^{-1} B - C\right)  - H^{-1} B
  \end{align*}
  where
  \begin{align*}
    A &= \dd[Y] h(x,y) \in \reals^{p \times m}\\
    B &= \dd[XY]^2 f(x, y) - \sum_{i = 1}^{p} \lambda_i \dd[XY]^2 h_i(x, y) \in \reals^{m \times n}\\
    C &= \dd[X] h(x,y) \in \reals^{p \times n}\\
    H &= \dd[YY]^2 f(x, y) - \sum_{i = 1}^{p} \lambda_i \dd[YY]^2 h_i(x, y)  \in \reals^{m \times m}
  \end{align*}
  and $\lambda \in \reals^p$ satisfies $\lambda\transpose \!A = \dd[Y] f(x, y)$.
  \label{prop:main}
\end{proposition}
\begin{proof}
  By the method of Lagrange multipliers~\cite{Bertsekas:1982} we can form the Lagrangian
  \begin{align}
    \Ell(x, y, \lambda) &= f(x, y) - \sum_{i = 1}^{p} \lambda_i h_i(x, y).
  \end{align}
  Assume $y$ is optimal for a fixed input $x$. Since $\dd[Y] h(x, y)$ is
  full rank we have that $y$ is a regular point. Then there exists a
  $\lambda$ such that the Lagrangian is stationary at the point $(y,
  \lambda)$. Here both $y$ and $\lambda$ are understood to be functions
  of $x$. Thus
  \begin{align}
    \begin{bmatrix}
      \left( \dd[Y] f(x, y) - \sum_{i = 1}^{p} \lambda_i \dd[Y] h_i(x, y) \right)\transpose \\
      h(x, y)
    \end{bmatrix}
    &= \zeros_{m+p}
    \label{eqn:lagrange_gradient}
  \end{align}
  where the first $m$ rows are from differentiating $\Ell$ with
  respect to $y$ (that is, $\dd[Y] \Ell\transpose$) and the last $p$ rows are
  from differentiating $\Ell$ with respect to $\lambda$ (that is,
  $\dd[\Lambda] \Ell\transpose$).
  
  Now observe that at the optimal point $y$ we have that either
  $\dd[Y] f(x, y) = \zeros_{1 \times m}$, that is, the optimal point
  of the unconstrained problem automatically satisfies the
  constraints, or $\dd[Y] f(x, y)$ is non-zero and orthogonal to the
  constraint surface defined by $h(x, y) = 0$. In the first case we
  can simply set $\lambda = \zeros_{p}$. In the second case we have
  (from the first row in \eqnref{eqn:lagrange_gradient})
  \begin{align}
    \dd[Y] f(x, y) &= \sum_{i = 1}^{p} \lambda_i \dd[Y] h_i(x, y) = \lambda\transpose \! A
  \end{align}
  for $A$ defined above.
  
  Now, differentiating the gradient of the Lagrangian with respect to $x$ we have
  \begin{align}
    \dd
    \begin{bmatrix}
      (\dd[Y] f(x, y))\transpose - \sum_{i = 1}^{p} \lambda_i (\dd[Y] h_i(x, y))\transpose \\
      h(x, y) \\
    \end{bmatrix}
    = \zeros
  \end{align}
  For the first row this gives
  \begin{align}
    &\dd[XY]^2 f + \dd[YY]^2 f \dd y - \dd[Y] h\transpose \dd \lambda \nonumber\\
    & \qquad\quad - \sum_{i=1}^{p}\lambda_i\left( \dd[XY]^2 h_i + \dd[YY]^2 h_i \dd y\right) = \zeros_{m \times n}
  \end{align}
  and for the second row we have
  \begin{align}
    \dd[X] h + \dd[Y] h \dd y &= \zeros_{p \times n}.
    \label{eqn:Dh}
  \end{align}
  Therefore
  \begin{align}
    &\begin{bmatrix}
       \dd[YY]^2 f - \sum_{i = 1}^{p}\lambda_i \dd[YY]^2 h_i & -\dd[Y] h\transpose
       \\
       \dd[Y] h & \zeros_{p \times p}
     \end{bmatrix}
    \begin{bmatrix}
      \dd y \\
      \dd \lambda
    \end{bmatrix} \nonumber\\
    &\qquad\quad +
    \begin{bmatrix}
      \dd[XY]^2 f - \sum_{i = 1}^{p}\lambda_i \dd[XY]^2 h_i
      \\
      \dd[X] h
    \end{bmatrix}
    = \zeros_{(m+p) \times n}
  \end{align}
  where all functions are evaluated at $(x, y)$.
  We can now solve by variable elimination \cite{Boyd:2004} to get
  \begin{align}
    \dd \lambda(x) &=
    \left(A H^{-1} A\transpose \right)^{-1} \left( A H^{-1}B - C\right)
    \\
    \dd y(x) &=
    H^{-1} \! A\transpose \!\! \left( A H^{-1} \! A\transpose \right)^{\!-1} \!\!\left(A H^{-1} B - C\right)  - H^{-1} B
  \end{align}
  with $A$, $B$, $C$, and $H$ as defined above.
\end{proof}

The Lagrange multipliers $\lambda$ in \propref{prop:main} are often
made available by the method used to solve for $y(x)$, e.g., in the
case of primal-dual methods for convex optimization problems. Where
it is not provided by the solver it can be computed explicitly solving
$\lambda\transpose \!A = \dd[Y] f(x, y)$, which has unique analytic
solution
\begin{align}
  \lambda &= (AA\transpose)^{-1}A\,(\dd[Y] f)\transpose.
\end{align}
Thus given optimal $y$ we can find $\lambda$.

Although the result looks expensive to compute, much of the
computation is shared since $H^{-1}A\transpose$ is independent of the
coordinate of $x$ for which the gradient is being computed. Moreover,
$H$ need only be factored once and then reused in computing $H^{-1}B$ for
each input dimension, that is, column of $B$. And for many problems
$H$ has structure that can be further exploited to speed calculation
of $\dd y$.

Inequality constrained problems encompass a richer class of deep
declarative nodes. \citename{Gould:TR2016} addressed the question of
computing gradients for these problems by approximating the inequality
constraints with a logarithmic barrier~\cite{Boyd:2004} and leveraging
the result for unconstrained problems. However, it is also possible to
calculate the gradient without resorting to approximation by
leveraging the KKT optimality conditions~\cite{Bertsekas:2004}
as has been shown in previous work for the special case
of convex optimization problems~\cite{Amos:ICML2017, Agrawal:TR2019}.
Here we present a result for the more general non-convex case.

We first observe that for a smooth objective function $f$ the optimal
solution $y(x)$ to the inequality constrained problem in \eqnref{eqn:dde_inequ_const}
is also a local (possibly global) minimum of the problem with any or all inactive
inequality constraints removed. Moreover, the derivative $\dd y(x)$ is the same
for both problems. As such, without loss of generality, we can assume that all
inequality constraints are active.

\begin{proposition}[Inequality Constrained]
	Consider functions $f: \reals^n \times \reals^m \to \reals$, $h: \reals^n \times
	\reals^m \to \reals^p$ and $g: \reals^n \times \reals^m \to \reals^q$. Let
	\begin{align*}
	\arraycolsep=2pt
	\begin{array}{rlll}
	y(x) \in& {\textstyle\argmin_{u \in \reals^m}} & f(x, u) &\\
	& \text{subject to} & h_i(x, u) = 0, & i = 1, \ldots, p \\
	& & g_i(x, u) \leq 0, & i = 1, \ldots, q.
	\end{array}
	\end{align*}
	Assume that $y(x)$ exists, that $f$, $h$ and $g$ are second-order 
	differentiable in the neighborhood of $(x, y(x))$,
	and that all inequality constraints are active at $y(x)$. Let
	$\tilde{h} = [h_1, \ldots, h_p, g_1, \ldots, g_q]$
	and assume $\text{rank}(\dd[Y] \tilde{h}(x, y)) = p + q$.
	Then for $H$ non-singular
	\begin{align*}
	\dd y(x) &= H^{-1} \! A\transpose \!\! \left( A H^{-1} \! A\transpose \right)^{\!-1} \!\!\left(A H^{-1} B - C\right)  - H^{-1} B
	\end{align*}
	where
	\begin{align*}
	A &= \dd[Y] \tilde{h}(x, y) \in \reals^{(p+q) \times m}\\
	B &= \dd[XY]^2 f(x, y) - \sum_{i = 1}^{p + q} \lambda_i \dd[XY]^2 \tilde{h}_i(x, y) \in \reals^{m \times n}\\
	C &= \dd[X] \tilde{h}(x,y) \in \reals^{(p+q) \times n}\\
	H &= \dd[YY]^2 f(x, y) - \sum_{i = 1}^{p + q} \lambda_i \dd[YY]^2 \tilde{h}_i(x, y)  \in \reals^{m \times m}
	\end{align*}
	and $\lambda \in \reals^{p+q}$ satisfies $\lambda\transpose A = \dd[Y] f(x, y)$ with
	$\lambda_i \leq 0$ for $i = p+1, \ldots, p+q$. The gradient $\dd y(x)$ is one-sided
        whenever there exists an $i > p$ such that $\lambda_i = 0$.
	\label{prop:main_ineq}
\end{proposition}

\begin{proof}
  The existence of $\lambda$ and non-negativity of $\lambda_i$ for $i = p+1, \ldots, p+q$ comes
  from the fact that $y(x)$ is a regular point (because $\dd[Y] \tilde{h}(x, y)$ is full rank).
  The remainder of the proof follows \propref{prop:main}. One-sidedness of $\dd y(x)$ results
  from the fact that $\lambda_i(x)$, for $i = p+1, \ldots, p+q$, is not differentiable at zero.
\end{proof}

The main difference between the result for inequality constrained
problems (\propref{prop:main_ineq}) and the simpler equality
constrained case (\propref{prop:main}) is that the gradient is
discontinuous at points where any of the Lagrange multipliers for
(active) inequality constraints are zero. We illustrate this by
analysing a deep declarative node with a single inequality constraint,
\begin{align}
\arraycolsep=2pt
\begin{array}{rll}
y \in & \argmin_{u \in \reals^m} & f(x, u) \\
& \text{subject to} & g(x, u) \leq 0
\end{array}
\label{eqn:dde_ineq_simple}
\end{align}
where both $f$ and $g$ are smooth.

Consider three scenarios (see \figref{fig:ineqconst}).
First, if the constraint is inactive at solution $y$, that is, $g(x, y) < 0$,
then we must have $\dd[Y] f(x, y) = 0$ and can
take $\dd y(x)$ to be the same as for the unconstrained case. Second, if the
constraint is active at the solution, that is, $g(x, y) = 0$ but $\dd[Y] f(x, y) \neq 0$
then in must be that $\lambda \neq 0$, and from \propref{prop:main_ineq}, the gradient $\dd y(x)$ is the same as if
we had replaced the inequality constraint with an equality.
Last, if the constraint is active at $y$ and $\dd[Y] f(x, y) = 0$ then $\dd y(x)$ is
undefined. In this last scenario we can choose either the unconstrained or
constrained gradient in order to obtain an update direction for back-propagation
similar to how functions such as rectified linear units are treated in standard
deep learning models.

\begin{figure}
	\centering
	\includegraphics[]{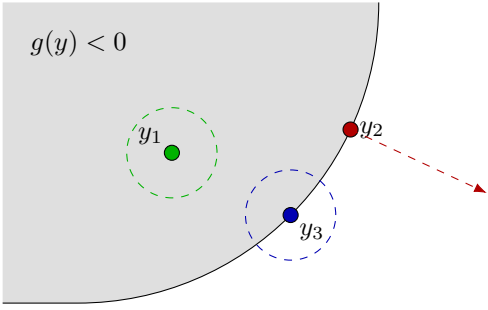}
	\caption{
		Illustration of different scenarios for the solution to inequality constrained deep declarative nodes.
		In the first scenario ($y_1$) the solution is a local minimum strictly satisfying the constraints. In
		the second scenario ($y_2$) the solution is on the boundary of the constraint set with the negative
		gradient of the objective pointing outside of the set. In the third scenario ($y_3$) the solution
		is on the boundary of the constraint set and is also a local minimum.}
	\label{fig:ineqconst}
\end{figure}

\subsection{Simpler equality constraints}
\label{sec:corrolaries}

Often there is only a single fixed equality constraint that does not
depend on the inputs---we consider such problems in
\secref{sec:projection}---or the equality constraints are all linear.
The above result can be specialized for these cases.

\begin{corollary}
Consider functions $f: \reals^n \times \reals^m \to \reals$ and $h: \reals^m \to \reals$. Let
\begin{align*}
\arraycolsep=2pt
\begin{array}{rll}
  y(x) \in& \argmin_{u \in \reals^m} & f(x, u) \\
  & \text{subject to} & h(u) = 0.
\end{array}
\end{align*}
Assume that $y(x)$ exists, that $f(x, u)$ and $h(u)$ are
second-order differentiable in the neighborhood of $(x, y(x))$
and $y(x)$, respectively,
and that $\dd[Y] h(y(x)) \neq 0$.
Then
\begin{align*}
\dd y(x) &= \left(\frac{H^{-1}a a\transpose H^{-1}}{a\transpose H^{-1}a} - H^{-1}\right) B
\end{align*}
where
\begin{align*}
a &= (\dd[Y] h(y))\transpose \in \reals^m, \\
B &= \dd[XY]^2 f(x, y) \in \reals^{m \times n}, \\
H &= \dd[YY]^2 f(x, y) - \lambda \dd[YY]^2 h(y) \in \reals^{m \times m}, \text{ and} \\
\lambda &= \left( \frac{\partial h}{\partial y_i}(y)\right) ^{-1} \frac{\partial f}{\partial y_i}(x, y) \in \reals \text{ for any } \frac{\partial h}{\partial y_i}(y) \neq 0.
\end{align*}
\label{cor:indep_x}
\end{corollary}
\begin{proof}
  Follows from \propref{prop:main} with $p=1$, $\dd[X] h \equiv \zeros_{1 \times n}$ and
  $\dd[XY]^2 h \equiv \zeros_{m \times n}$.
\end{proof}

\begin{observation}
  When the constraint function is independent of $x$ we have $\dd y(x) \perp \dd[Y] h(y)$, from \eqnref{eqn:Dh} with $\dd[X] h = 0$. In other
  words, $y$ can only change (as a function of $x$) in directions that maintain the constraint
  $h(y) = 0$.
\end{observation}

Given this simpler form of the gradient for declarative nodes with
a single equality constraint, one might be tempted to naively combine multiple equality
constraints into a single constraint, for example $\tilde{h}(x, u)
\triangleq \sum_{i=1}^{p} h_i^2(x, u) = 0$ (or any other function of
the $h_i$'s that is identically zero if any only if the $h_i$'s are all
zero). However, this will not work as it violates the assumptions of the
method of Lagrange multipliers, namely that $\dd[Y] \tilde{h}(x, y) \neq 0$ at the
optimal point.

The following result is for the common case of multiple fixed linear
equality constraints that do not depend on the inputs (which has also
been reported previously~\cite{Gould:TR2016}).

\begin{corollary}
Consider a function $f: \reals^n \times \reals^m \to \reals$ and
let $A \in \reals^{p \times m}$ and $d \in \reals^p$ with $\text{rank}(A) = p$
define a set of $p$ under-constrained linear equations $Au = d$. Let
\begin{align*}
\begin{alignedat}{2}
y(x) \in & \enskip {\textstyle\argmin_{u \in \reals^m}} && \enskip f(x, u) \\
& \enskip \text{subject to} && Au = d.
\end{alignedat}
\end{align*}
Assume that $y(x)$ exists and that $f(x, u)$ is
second-order differentiable in the neighborhood of $(x, y(x))$.
Then
\begin{align*}
\dd y(x) &= \left(H^{-1}A\transpose (AH^{-1}A\transpose)^{-1}\!AH^{-1} - H^{-1}\right) B
\end{align*}
where $H = \dd[YY]^2 f(x, y)$ and $B = \dd[XY]^2 f(x, y)$.
\label{cor:affine}
\end{corollary}

\begin{proof}
	Follows from \propref{prop:main} with $h(x, u) \triangleq Au - d$ so that $\dd[X] h \equiv \zeros_{p \times n}$,
	$\dd[XY]^2 h \equiv \zeros_{m \times n}$ and $\dd[YY]^2 h \equiv \zeros_{m \times m}$.
\end{proof}

\subsection{Geometric interpretation}
\label{sec:geometric_interp}

Consider, for simplicity, a declarative node with
scalar input $x$ ($n=1$) and a single equality constraint ($p=1$).
An alternative form for the gradient from \propref{prop:main},
which lends itself to geometric interpretation
as shown in Figure~\ref{fig:gradient_geometry}, is given by
\begin{align}
\dd y(x) &= H^{-\frac{1}{2}} \left( \tilde{b} - \left( \hat{a}\transpose \tilde{b}\right)  \hat{a} + \frac{c}{\|\tilde{a}\|} \hat{a} \right)
\end{align}
where
$\tilde{a} = -H^{-\frac{1}{2}} a$,
$\tilde{b} = -H^{-\frac{1}{2}} b$ and
$\hat{a} = \|\tilde{a}\|^{-1}\tilde{a}$.
First, consider the case when the constraints do not depend on the input ($c = 0$).
The unconstrained derivative $\dd y_\text{unc}(x) = H^{-\frac{1}{2}}\tilde{b}$
(equality when $\lambda = 0$)
is corrected by removing the component perpendicular to the constraint surface.
That is, the unconstrained gradient is projected to the tangent plane of the constraint surface,
as shown in Figure~\ref{fig:gradient_geometry}.
The gradient is therefore parallel to the constraint surface, ensuring that gradient descent with an infinitesimal step size does not push the solution away from the constraint surface.
Thus, the equation encodes the following procedure: (i)~pre-multiply the unconstrained gradient vector $\dd y_\text{unc}(x)$ by $H^{\frac{1}{2}}$; (ii)~project onto the linearly transformed (by $H^{-\frac{1}{2}}$) tangent plane of the constraint surface; and (iv)~pre-multiply by $H^{-\frac{1}{2}}$.
Observe that the equation compensates for the curvature $H$ before the projection.
The same intuition applies to the case where the constraints depend on the inputs ($c \neq 0$), with an additional bias term
accounting for how the constraint surface changes with the input $x$.

\begin{figure}
	\centering
	\includegraphics[]{gradient_geometry.tikz}
	\caption{Geometry of the gradient for an equality constrained optimization problem. The unconstrained gradient $\dd y_\text{unc}(x)$ is corrected to ensure that the solution remains on the constraint surface after gradient descent with an infinitesimal step size.}
	\label{fig:gradient_geometry}
\end{figure}

\subsection{Relationship to conventional neural networks}

Deep declarative networks subsume traditional
feed-forward and recurrent neural networks in the sense that any layer
in the network that explicitly defines its outputs in terms of a
differentiable function of its inputs can be reduced to a declarative
processing node. A similar observation was made by \citename{Amos:ICML2017},
and the notion is formalized in the following proposition.

\begin{proposition}
Let $\tilde{f}: \reals^n \to \reals^m$ be an explicitly defined
differentiable forward processing function for a neural network
layer. Then we can alternatively define the behavior of the layer
declaratively as
\begin{align*}
	y &= \argmin_{u \in \reals^m} \frac{1}{2} \|u - \tilde{f}(x)\|^2.
\end{align*}
\end{proposition}

\begin{proof}
  The objective function $f(x, u) = \frac{1}{2} \|u -
  \tilde{f}(x)\|^2$ is strongly convex. Differentiating it with
  respect to $u$ and setting to zero we get $y = \tilde{f}(x)$. By
  \propref{prop:unconstrained} we have $H = I$ and $B = -\dd[X]
  \tilde{f}(x)$, giving $\dd y(x) = \dd \tilde{f}(x)$.
\end{proof}

Of course, despite the appeal of having a single mathematical framework
for describing processing nodes, there is no reason to do this in practice
since imperative and declarative nodes can co-exist in a network.

\textbf{Composability.}
One of the key design principles in deep learning is that models should define
their output as the composition of many simple functions (such as $y = \tilde{f}(\tilde{g}(x))$).
Just like conventional neural networks 
deep declarative networks can also be composed of many optimization problems
arranged in levels as the following example shows:
\begin{align}
  \arraycolsep=2pt
  \begin{array}{rll}
      y \in& \argmin_{u} & f(z, u) \\
      & \text{subject to} & z \in \argmin_{v} g(x, v).
  \end{array}
  \label{eqn:composition}
\end{align}

Here the gradients combine by the chain rule of differentiation as one might expect:
\begin{align}
\begin{split}
  \dd z(x) &= -H_g^{-1} B_g(x, z) \\
  \dd y(z) &= -H_f^{-1} B_f(z, y) \\
  \dd y(x) &= \dd y(z) \dd z(x)
\end{split}
  \label{eqn:grads}
\end{align}
where $H_g = \dd[ZZ]^2 g(x, z)$, etc. Here, during back propagation in computing
$\dd J(x)$ from $\dd J(y)$ for training objective $J$, it may be more efficient to
first compute $\dd J(z)$ as
\begin{align}
	\dd J(x) &= \underbrace{\left(\dd J(y) \dd y(z)\right)}_{\dd J(z)} \dd z(x),
\end{align}
which we will discuss further in \secref{sec:vec_jac_prod}.

\subsection{Extensions and discussion}
\label{sec:extensions}

\subsubsection{Feasibility problems}
\label{sec:feasibility}

Feasibility problems, where we simply seek a solution $y(x)$ that satisfies a set
of constraints (themselves functions of $x$), can be formulated as optimization
problems~\cite{Boyd:2004}, and hence declarative nodes. Here the objective is
constant (or infinite for an infeasible point) but the problem is more naturally
written as
\begin{align}
  \arraycolsep=2pt
    \begin{array}{lll}
	  \text{find} & u \\
      \text{subject to} &  h_i(x, u) = 0, & i = 1, \ldots, p \\
      & g_i(x, u) \leq 0, & i = 1, \ldots, q.
\end{array}
\label{eqn:feasible}
\end{align}
The gradient can be derived by removing inactive inequality constraints
and maintaining the invariant $\tilde{h}(x, y(x)) = 0$ where
$\tilde{h} = [h_1, \ldots, h_p, g_1, \ldots, g_q]$, i.e., following
the constraint surface. Implicit differentiation gives
\begin{align}
  A \dd y(x) + C &= 0
\end{align}
where $A = \dd[Y] \tilde{h}(x, y)$ and $C = \dd[X] \tilde{h}(x, y)$ as defined
in \propref{prop:main_ineq}. This is the same set of linear equations as the
second row of the Lagrangian in our proof of \propref{prop:main}, and has a unique
solution whenever $\text{rank}(A) = m$.%
\footnote{In \propref{prop:main} we assumed that $\text{rank}(A) = p \leq m$. Typically
	we have $p < m$ and so $A \dd y(x) = -C$ is under-determined, and the curvature of the objective resolves $\dd y(x)$.}

\subsubsection{Non-smooth objective and constraints}
\label{sec:nonsmooth}

Propositions~\ref{prop:main} and \ref{prop:main_ineq} showed how to compute derivatives of the solution to parametrized equality
constrained optimization problems with second-order differentiable objective and constraint functions
(in the neighborhood of the solution). However, we
can also define declarative nodes where the objective and constraints are non-smooth (that is,
non-differentiable at some points). To enable back-propagation learning with such declarative
nodes all we require is a local descent direction (or Clarke generalized
gradient~\cite{ClarkeTAMS:1975}). This is akin to the approach taken
in standard deep learning models with non-smooth activation functions such as the rectified linear
unit, $y = \max\{x, 0\}$, which is non-differentiable at $x = 0$.%
\footnote{Incidentally, the (elementwise) rectified linear unit can be defined declaratively as
$y = \argmin_{u \in \reals^{n}_{+}} \frac{1}{2} \|u - x\|_2^2$.}
We explore this in practice when we consider robust pooling and projecting
onto the boundary of $L_p$-balls in our illustrative examples and
experiments below.

\subsubsection{Non-regular solutions}
\label{sec:nonregular}

The existence of the Lagrange multipliers $\lambda$ (and their uniqueness)
in our results above was guaranteed by our assumption that $y(x)$
is a regular point. Moreover, the fact that $A = \dd[Y] h$ is full rank (and $H$ non-singular)
ensures that $\dd y(x)$ can be computed. Under other constraint qualifications, such as
Slater's condition~\cite{Boyd:2004}, the $\lambda$ are also guaranteed to exist but may not
be unique. Moreover, we may have $\text{rank}(A) \neq p + q$. If $p + q > m$ but $A$ is full rank,
which can happen if multiple constraints are active at $y(x)$, then we can directly solve the
overdetermined system of equations $A \dd y(x) = -C$.
On the other hand, if $A$ is rank deficient then we have
redundant constraints and must reduce $A$ to a full rank matrix by removing linearly dependent rows (and fixing
the corresponding $\lambda_i$ to zero) or resort to using a pseudo-inverse to approximate $\dd y(x)$.
This can happen if two or more (nonlinear inequality) constraints are active at $y(x)$ and tangent to each other at that point.

\subsubsection{Non-unique solutions}
\label{sec:nonunique}

In the development of deep declarative nodes we have made no assumption about
the uniqueness of the solution to the optimization problem (i.e., output of the node).
Indeed, many solutions may exist and our gradients are still valid. Needless to say, 
it is essential to use the same value of $y$ in the backward pass
as was computed during the forward pass (which is automatically handled in modern deep
learning frameworks by caching the forward computations). But there are other
considerations too.

Consider for now the
unconstrained case (\eqnref{eqn:dde_unconst}). If $\dd[YY]^2 f(x, y) \succ 0$
then $f$ is strongly convex in the neighborhood of $y$. Hence $y$ is an isolated
minimizer (that is, a unique minimizer within its neighborhood), and the gradient
derived in \propref{prop:unconstrained} holds for all such points. Nevertheless,
when performing parameter learning in a deep declarative network it is important
to be consistent in solving optimization problems with multiple minima to avoid
oscillating between solutions as we demonstrate with the following pathological
example. Consider
\begin{align}
  \arraycolsep=2pt
  \begin{array}{rll}
    y \in & \argmin_{u \in \reals} & 0 \\
    & \text{subject to} & (y - 1)^2 - x^2 = 0
  \end{array}
  \label{eqn:patho}
\end{align}
which has solution $y = 1 \pm x$ and, by inspection,
\begin{align}
  \dd y &= \begin{cases}
    -1, & \text{for $y = 1 - x$} \\
    +1, & \text{for $y = 1 + x$}
  \end{cases}
\end{align}
Depending on the choice of optimal solution the gradient of the loss being
propagated backwards through the node will point in opposite directions, which
is problematic for end-to-end learning.

\textbf{Singular Hessians.}
Now, if $\dd[YY]^2 f(x, y)$ is singular then $y$ may not be an isolated
minimizer. This can occur, for example, when $y$ is an
over-parametrized descriptor of some physical property such as an
unnormalized quaternion representation of a 3D rotation. In such
cases we cannot use \propref{prop:unconstrained} to find the
gradient. In fact the gradient is undefined. There are three strategies
we can consider for such points. First, reformulate the problem to
use a minimal parametrization or introduce constraints to remove
degrees of freedom thereby making the solution unique. Second, make the objective function
strongly convex around the solution, for example by adding the proximal term
$\frac{\delta}{2} \|u - y\|^2$ for some small $\delta > 0$. Third, use a
pseudo-inverse to solve the linear equation $\dd[YY]^2 f(x, y) \Delta y = -\dd[XY]^2(x,
y)$ for $\Delta y$ and take $\Delta y$ as a descent direction.
Supposing $H^\dagger$ is a pseudo-inverse of $\dd[YY]^2 f(x, y)$,
we can compute an entire family of solutions as
\begin{align}
\Delta y \in \{ -H^\dagger B + (I - H^\dagger H)Z \mid Z \in \reals^{m \times n} \}
\end{align}
where $H = \dd[YY]^2 f(x, y)$ and $B = \dd[XY]^2 f(x, y)$.

As a toy example, motivated by the quaternion representation,
consider the following problem that aligns the output vector with an
input vector $x \neq 0$ in $\reals^4$,
\begin{align}
  y &\in \argmin_{u \in \reals^4} f(x, u) \quad \text{with } f(x, u) \triangleq \frac{-x\transpose u}{\|u\|_2}
\end{align}
which has solution $y = \alpha x$ for arbitrary $\alpha > 0$. Here we have
$\dd[YY]^2 f(x, y) = \alpha^{-2} \|x\|_2^{-1}\left(I - \|x\|_2^{-2} xx\transpose \right)$,
which is singular (by the matrix inversion lemma~\cite{Golub:1996}). Fixing one
degree of freedom resolves the problem, which in this case is easily
done by forcing the output vector to be normalized:
\begin{align}
  \arraycolsep=2pt
  \begin{array}{rll}
    y \in & \argmin_{u \in \reals^4} & -x\transpose u \\
    & \text{subject to} & \|u\|_2 = 1.
  \end{array}
\end{align}
Alternatively we can compute the Moore--Penrose pseudo-inverse~\cite{Meyer:JSTOR1973}
of $\dd[YY]^2 f(x,y)$ to obtain
\begin{align}
  \dd y &= \alpha \left(I - \frac{1}{\|x\|_2^2} x x\transpose\right)
\end{align}
which is the same gradient as the constrained case where the solution is fixed
to have
$\alpha = 1/\|x\|_2$
for this problem.

\section{Illustrative Examples}
\label{sec:examples}

In this section we provide examples of unconstrained and constrained
declarative nodes that illustrate the theory developed above. For each
example we begin with a standard operation in deep learning models,
which is usually implemented as an imperative node. We show how the
operations can be equivalently implemented in a declarative framework,
and then generalize the operations to situations where an explicit
implementation is not possible, but where the operation results in
more desirable model behavior.

\subsection{Robust pooling}
\label{sec:robust_pooling}

Average (or mean) pooling is a standard operation in deep learning
models where multi-dimensional feature maps are averaged over one or
more dimensions to produce a summary statistic of the input data. For
the one-dimensional case, $x \in \reals^n \mapsto y \in \reals$, we
have
\begin{align}
  y &= \frac{1}{n} \sum_{i=1}^{n} x_i
\end{align}
and
\begin{align}
  \dd y &= \left[\frac{\partial y}{\partial x_1}, \ldots, \frac{\partial y}{\partial x_n}\right] = \frac{1}{n} \ones\transpose.
\end{align}
As a declarative node the mean pooling operator is
\begin{align}
  y &\in \argmin_{u \in \reals} \sum_{i=1}^{n} \frac{1}{2} (u - x_i)^2
\end{align}

While the solution, and hence gradients, can be expressed in closed form,
it is well-known that as a summary statistic the mean is very sensitive to
outliers. We can make the statistic more robust by replacing the quadratic
penalty function $\phi^{\text{quad}}(z) = \frac{1}{2} z^2$ with one that is
less sensitive to outliers. The pooling operation can then be generalized
as
\begin{align}
  y &\in \argmin_{u \in \reals} \sum_{i=1}^{n} \phi(u - x_i; \alpha)
\end{align}
where $\phi$ is a penalty function taking scalar argument and
controlled by parameter $\alpha$.

For example, the Huber penalty function defined as
\begin{align}
  \phi^{\text{huber}}(z; \alpha) &= \begin{cases}
    \frac{1}{2} z^2 & \text{for $|z| \leq \alpha$} \\
    \alpha(|z| - \frac{1}{2}\alpha) & \text{otherwise.}
  \end{cases}
  \label{eqn:huber}
\end{align}
is more robust to outliers. The Huber penalty is convex and hence the
pooled value can be computed efficiently via Newton's method or
gradient descent~\cite{Boyd:2004}. However, no closed-form solution
exists. Moreover, the solution set may be an interval of points. The
pseudo-Huber penalty~\cite{Hartley:2004},
\begin{align}
  \phi^{\text{pseudo}}(z; \alpha) &= \alpha^2 \left( \sqrt{1 + \left(\frac{z}{\alpha}\right)^2} - 1 \right)
  \label{eqn:pseudo_huber}
\end{align}
has similar behaviour to the Huber but is strongly convex so the
optimal solution (pooled value) is unique.

The Welsch penalty function~\cite{Dennis:CSSC78} is even more robust to
outliers flattening out to a fixed cost as $|z| \to \infty$,
\begin{align}
  \phi^{\text{welsch}}(z; \alpha) &= 1 - \exp\left(-\frac{z^2}{2 \alpha^2} \right)
  \label{eqn:welsch}
\end{align}
However it is non-convex and so obtaining the solution $y(x)$ is non-trivial.
Nevertheless, \emph{given a solution} we can use \propref{prop:unconstrained} to compute
a gradient for back-propagation parameter learning.

Going one step further we can defined the truncated quadratic penalty function,
\begin{align}
\phi^{\text{trunc}}(z; \alpha) &= \begin{cases}
\frac{1}{2} z^2 & \text{for $|z| \leq \alpha$} \\
\frac{1}{2} \alpha^2 & \text{otherwise.}
\end{cases}
\label{eqn:trunc_quad}
\end{align}
In addition to being non-convex, the function is also non-smooth. The
solution amounts to finding the maximal set of values all within some
fixed distance $\alpha$ from the mean. The objective $f(x, u) =
\sum_{i=1}^{n} \phi^{\text{trunc}}(u - x_i; \alpha)$ is not
differentiable at points where there exists an $i$ such that $|y(x) -
x_i| = \alpha$, in the same way that the rectified linear unit is
non-differentiable at zero. Nevertheless, we can still compute a
gradient almost everywhere (and take a one-sided gradient at
non-differentiable points).

The various penalty functions are depicted in
\tabref{tab:robust_pooling}. When used to define a robust pooling
operation there is no closed-form solution for all but the quadratic
penalty. Yet the gradient of the solution with respect to the input
data for all robust penalties can be calculated using
\propref{prop:unconstrained} as summarized in the table.

\afterpage{
	\begin{landscape}
		\begin{table}
  \centering
  \newcolumntype{C}{>{\centering\arraybackslash}X}
  \renewcommand*{\arraystretch}{1.3}
  \begin{tabularx}{\linewidth}{@{}l | C C C C C@{}}
  	\multicolumn{1}{l}{}
    & {\sc Quadratic} & {\sc Pseudo-Huber} & {\sc Huber} & {\sc Welsch} & {\sc Trunc. Quad.} \\
    \hline
    \rule{0pt}{2.0\normalbaselineskip}
    $\phi(z; \alpha)$
    &
    $\frac{1}{2} z^2$
    &
	$\alpha^2 \left( \sqrt{1 + \left( \frac{z}{\alpha}\right)^2 } - 1 \right)$
    &
    \(
    \begin{cases}
      \frac{1}{2} z^2 & \text{for $|z| \leq \alpha$} \\
      \alpha(|z| - \frac{1}{2}\alpha) & \text{otherwise.}
    \end{cases}
    \)
    &
    $1 - \exp\left(-\frac{z^2}{2 \alpha^2} \right)$
    &
    \(
    \begin{cases}
      \frac{1}{2} z^2 & \text{for $|z| \leq \alpha$} \\
      \frac{1}{2} \alpha^2 & \text{otherwise.}
    \end{cases}
    \)
    \\[12pt]
    &
    \begin{tikzpicture}[xscale=0.25, yscale=0.25]
      \draw[-{Latex[]}] (-4, 0) -- (4, 0) node[below] {$z$};
      \draw[-{Latex[]}] (0, -1) -- (0, 5) node[right] {$\phi$};

      \draw[domain=-3:3,smooth,variable=\x,black] plot ({\x}, {0.5*\x*\x});
    \end{tikzpicture}
    &
    \begin{tikzpicture}[xscale=0.25, yscale=0.25]
      \draw[-{Latex[]}] (-4, 0) -- (4, 0) node[below] {$z$};
      \draw[-{Latex[]}] (0, -1) -- (0, 5) node[right] {$\phi$};

      \draw[domain=-3:3,smooth,variable=\x,black] plot ({\x}, {2.0*(sqrt(1 + \x*\x) - 1)});
    \end{tikzpicture}
    &
    \begin{tikzpicture}[xscale=0.25, yscale=0.25]
      \draw[-{Latex[]}] (-4, 0) -- (4, 0) node[below] {$z$};
      \draw[-{Latex[]}] (0, -1) -- (0, 5) node[right] {$\phi$};

      \draw[domain=-1:1,smooth,variable=\x,black] plot ({\x}, {1.5 * (0.5*\x*\x)});
      \draw[domain=-3:-1,smooth,variable=\x,black] plot ({\x}, {1.5 * (abs(\x) - 0.5)});
      \draw[domain=1:3,smooth,variable=\x,black] plot ({\x}, {1.5 * (abs(\x) - 0.5)});
    \end{tikzpicture}
    &
    \begin{tikzpicture}[xscale=0.25, yscale=0.25]
      \draw[-{Latex[]}] (-4, 0) -- (4, 0) node[below] {$z$};
      \draw[-{Latex[]}] (0, -1) -- (0, 5) node[right] {$\phi$};

      \draw[domain=-4:4,smooth,variable=\x,black] plot ({\x}, {2.0 * (1.0 - exp(-0.5*\x*\x))});
    \end{tikzpicture}
    &
    \begin{tikzpicture}[xscale=0.25, yscale=0.25]
      \draw[-{Latex[]}] (-4, 0) -- (4, 0) node[below] {$z$};
      \draw[-{Latex[]}] (0, -1) -- (0, 5) node[right] {$\phi$};

      \draw[domain=-1:1,smooth,variable=\x,black] plot ({\x}, {4.0 * 0.5*\x*\x});
      \draw[domain=-4:-1,smooth,variable=\x,black] plot ({\x}, {4.0 * 0.5});
      \draw[domain=1:4,smooth,variable=\x,black] plot ({\x}, {4.0 * 0.5});
    \end{tikzpicture}
    \\[12pt]
	&
	\begin{tikzpicture}[xscale=0.25, yscale=0.125]
	\draw[-{Latex[]}] (-6, 0) -- (6, 0) node[below] {$u$};
	\draw[-{Latex[]}] (0, -2) -- (0, 12) node[right] {$f$};

	\draw[domain=-3:5,smooth,variable=\x,red,dotted] plot ({\x-2}, {0.5*\x*\x});
	\draw[domain=-5:3,smooth,variable=\x,blue,dotted] plot ({\x+2}, {0.5*\x*\x});
	\draw[domain=-2.75:2.75,smooth,variable=\x,black] plot ({\x}, {0.5*(\x-2)*(\x-2) + 0.5*(\x+2)*(\x+2)});	
	\end{tikzpicture}
    &
    \begin{tikzpicture}[xscale=0.25, yscale=0.125]
	\draw[-{Latex[]}] (-6, 0) -- (6, 0) node[below] {$u$};
	\draw[-{Latex[]}] (0, -2) -- (0, 12) node[right] {$f$};

	\draw[domain=-3:7,smooth,variable=\x,red,dotted] plot ({\x-2}, {1.5*(sqrt(1 + \x*\x) - 1)});
	\draw[domain=-7:3,smooth,variable=\x,blue,dotted] plot ({\x+2}, {1.5*(sqrt(1 + \x*\x) - 1)});

	\draw[domain=-5:5,smooth,variable=\x,black] plot ({\x}, {1.5*(sqrt(1 + (\x-2)*(\x-2)) - 1 + sqrt(1 + (\x+2)*(\x+2)) - 1)});
	\end{tikzpicture}
    &
    \begin{tikzpicture}[xscale=0.25, yscale=0.125]
	\draw[-{Latex[]}] (-6, 0) -- (6, 0) node[below] {$u$};
	\draw[-{Latex[]}] (0, -2) -- (0, 12) node[right] {$f$};

	\draw[domain=-1:1,smooth,variable=\x,red,dotted] plot ({\x-2}, {1.5 * (0.5*\x*\x)});
	\draw[domain=-3:-1,smooth,variable=\x,red,dotted] plot ({\x-2}, {1.5 * (abs(\x) - 0.5)});
	\draw[domain=1:7,smooth,variable=\x,red,dotted] plot ({\x-2}, {1.5 * (abs(\x) - 0.5)});

	\draw[domain=-1:1,smooth,variable=\x,blue,dotted] plot ({\x+2}, {1.5 * (0.5*\x*\x)});
	\draw[domain=-7:-1,smooth,variable=\x,blue,dotted] plot ({\x+2}, {1.5 * (abs(\x) - 0.5)});
	\draw[domain=1:3,smooth,variable=\x,blue,dotted] plot ({\x+2}, {1.5 * (abs(\x) - 0.5)});

	\draw[domain=-4.5:-3,smooth,variable=\x,black] plot ({\x}, {1.5 * (abs(\x-2) - 0.5 + abs(\x+2) - 0.5)});
	\draw[domain=-3:-1,smooth,variable=\x,black] plot ({\x}, {1.5 * (0.5*(\x+2)*(\x+2) + abs(\x-2) - 0.5)});
	\draw[domain=-1:1,smooth,variable=\x,black] plot ({\x}, {1.5 * (abs(\x-2) - 0.5 + abs(\x+2) - 0.5)});
	\draw[domain=1:3,smooth,variable=\x,black] plot ({\x}, {1.5 * (0.5*(\x-2)*(\x-2) + abs(\x+2) - 0.5)});
	\draw[domain=3:4.5,smooth,variable=\x,black] plot ({\x}, {1.5 * (abs(\x-2) - 0.5 + abs(\x+2) - 0.5)});
	\end{tikzpicture}
    &
    \begin{tikzpicture}[xscale=0.25, yscale=0.25]
	\draw[-{Latex[]}] (-6, 0) -- (6, 0) node[below] {$u$};
	\draw[-{Latex[]}] (0, -1) -- (0, 6) node[right] {$f$};

	\draw[domain=-4:6,smooth,variable=\x,red,dotted] plot ({\x-2}, {2.25 * (1.0 - exp(-0.5*\x*\x))});
	\draw[domain=-6:4,smooth,variable=\x,blue,dotted] plot ({\x+2}, {2.25 * (1.0 - exp(-0.5*\x*\x))});
	\draw[domain=-6:6,smooth,variable=\x,black] plot ({\x}, {2.25 * (1.0 - exp(-0.5*(\x-2)*(\x-2)) + (1.0 - exp(-0.5*(\x+2)*(\x+2))))});
	\end{tikzpicture}
    &
    \begin{tikzpicture}[xscale=0.25, yscale=0.25]
	\draw[-{Latex[]}] (-6, 0) -- (6, 0) node[below] {$u$};
	\draw[-{Latex[]}] (0, -1) -- (0, 6) node[right] {$f$};

	\draw[domain=-1:1,smooth,variable=\x,red,dotted] plot ({\x-2}, {4.0 * 0.5*\x*\x});
	\draw[domain=-4:-1,smooth,variable=\x,red,dotted] plot ({\x-2}, {4.0 * 0.5});
	\draw[domain=1:6,smooth,variable=\x,red,dotted] plot ({\x-2}, {4.0 * 0.5});

	\draw[domain=-1:1,smooth,variable=\x,blue,dotted] plot ({\x+2}, {4.0 * 0.5*\x*\x});
	\draw[domain=-6:-1,smooth,variable=\x,blue,dotted] plot ({\x+2}, {4.0 * 0.5});
	\draw[domain=1:4,smooth,variable=\x,blue,dotted] plot ({\x+2}, {4.0 * 0.5});

	\draw[domain=-6:-3,smooth,variable=\x,black] plot ({\x}, {4.0 * (0.5 + 0.5)});
	\draw[domain=-3:-1,smooth,variable=\x,black] plot ({\x}, {4.0 * (0.5 + 0.5*(\x+2)*(\x+2))});
	\draw[domain=-1:1,smooth,variable=\x,black] plot ({\x}, {4.0 * (0.5 + 0.5)});
	\draw[domain=1:3,smooth,variable=\x,black] plot ({\x}, {4.0 * (0.5 + 0.5*(\x-2)*(\x-2))});
	\draw[domain=3:6,smooth,variable=\x,black] plot ({\x}, {4.0 * (0.5 + 0.5)});
	\end{tikzpicture}
    \\[12pt]
    & closed-form, convex, & convex, smooth, & convex, non-smooth ($C^1$), & non-convex, smooth, & non-convex, non-smooth ($C^0$),
    \\
    & smooth, unique solution & unique solution & non-isolated solutions & isolated solutions & isolated solutions
    \\[12pt]
    $\dd[YY]^2 f(x, y)$
    &
    $n$
    &
    $\sum_{i=1}^{n} \left( 1 + \left( \frac{y - x_i}{\alpha} \right) ^2 \right)^{-3/2}$
    &
    $\sum_{i=1}^{n} \ind{|y - x_i| \leq \alpha}$
    &
    $\sum_{i=1}^{n} \frac{\alpha^2 - (y - x_i)^2}{\alpha^4} \exp\left(-\frac{(y - x_i)^2}{2 \alpha^2} \right)$
    &
    $\sum_{i=1}^{n} \ind{|y - x_i| \leq \alpha}$
    \\[18pt]
    $\dd[XY]^2 f(x, y)$
    &
    $-\ones_n^T$
    &
    $\myvec{-\left( 1 + \left( \frac{y - x_i}{\alpha} \right) ^2 \right)^{-3/2}}^T$
    &
    $\myvec{-\ind{|y - x_i| \leq \alpha}}$
    &
    $\myvec{\frac{(y - x_i)^2 - \alpha^2}{\alpha^4} \exp\left(-\frac{(y - x_i)^2}{2 \alpha^2}\right)}^T$
    &
    $\myvec{-\ind{|y - x_i| \leq \alpha}}^T$
    \\[18pt]
    $\dd y(x)$
    &
    $\frac{1}{n} \ones_n^T$
    &
    \begin{tabular}{c}
	  $\myvec{\frac{w_i}{\sum_{j=1}^{n} w_j}}^T$ where \\[12pt]
	  $w_i = \left( 1 + \left( \frac{y - x_i}{\alpha} \right) ^2 \right)^{-3/2}$
    \end{tabular}
    &
    $\myvec{\frac{\ind{|y - x_i| \leq \alpha}}{\sum_{j=1}^{n} \ind{|y - x_j| \leq \alpha}}}^T$
    &
    \begin{tabular}{c}
      $\myvec{\frac{w_i}{\sum_{j=1}^{n} w_j}}^T$ where \\[12pt]
      $w_i = \frac{\alpha^2 - (y - x_i)^2}{\alpha^4} \exp\left(-\frac{(y - x_i)^2}{2 \alpha^2}\right)$
    \end{tabular}
    &
    $\myvec{\frac{\ind{|y - x_i| \leq \alpha}}{\sum_{j=1}^{n} \ind{|y - x_j| \leq \alpha}}}^T$
    \\[-6pt]
    \\
    \hline
  \end{tabularx}
  \caption{The gradient of the estimate for a robust mean over a
    vector of values for various penalty functions $\phi(z; \alpha)$ when
    it exists. The robust estimate is found as $y(x) = \argmin_{u \in
      \reals} f(x, u)$ with $f(x, u) = \sum_{i=1}^{n} \phi(u - x_i;
    \alpha)$. The first plot (row 2) shows the penalty function, the second plot (row 3)
    shows an example $f(x, u)$ for $x = (-2\alpha, 2\alpha) \in \reals^2$.
    Plots have been drawn at different scales to enhance visualization of shape differences between the penalty functions.
    Despite not being able to solve the problem in
    closed form for most penalty functions, given a solution the gradient
    $\dd y(x)$ can still be calculated. Notice the similarity in gradient
    forms due to the objective $f$ being composed of a sum of independent penalties,
    each penalty symmetric in $x$ and $y$. Moreover, the Huber and
    truncated quadratic have exactly the same gradient form (when it exists)
    even though their solutions may be different. Under some conditions
    Huber and Welsch may result in a zero $\dd[YY]^2 f$ and, hence, an undefined gradient.
    However, we did not see this in practice. More importantly, computation of the
    Welsch gradient is sensitive to numerical underflow and care should be taken to
    appropriately scale the $w_i$'s before dividing by their sum.}
  \label{tab:robust_pooling}
\end{table}
	\end{landscape}
}

\subsection{$\boldsymbol{L_p}$-sphere or $\boldsymbol{L_p}$-ball projection}
\label{sec:projection}

Euclidean projection onto an $L_2$-sphere, equivalent to $L_2$ normalization, is another standard operation in deep learning models.
For $x \in \reals^n \mapsto y \in \reals^n$, we have
\begin{align}
y = \frac{1}{\|x\|_2} x
\end{align}
and
\begin{align}
\dd y(x) = \frac{1}{\|x\|_2} \left( I - \frac{1}{\|x\|_2^2} xx\transpose \right).
\end{align}
As a declarative node the $L_2$-sphere projection operator is
\begin{align}
\arraycolsep=2pt
\begin{array}{rll}
y \in & \argmin_{u \in \reals^n} & \frac{1}{2} \|u - x\|_2^2 \\
& \text{subject to} & \|u\|_2 = 1.
\end{array}
\end{align}

While the solution, and hence gradients, can be expressed in closed-form,
it may be desirable to use other $L_p$-spheres or balls, for regularization,
sparsification, or to improve generalization performance \cite{Oymak:ICML2018}, as
well as for adversarial robustness.
The projection operation onto an $L_p$-sphere can then be generalized as
\begin{align}
\arraycolsep=2pt
\begin{array}{rll}
y_p \in & \argmin_{u \in \reals^n} & \frac{1}{2} \|u - x\|_2^2 \\
& \text{subject to} & \|u\|_p = 1
\end{array}
\end{align}
where $\|\cdot\|_p$ is the $L_p$-norm.

For example, projecting onto the $L_1$-sphere has no closed-form solution,
however \citename{Duchi:ICML08} provide an efficient $O(n)$ solver.
Similarly, projecting onto the $L_\infty$-sphere has an efficient (trivial) solver.
Given a solution from one of these solvers, we can use \propref{prop:main} to
compute a gradient.
For both cases, the constraint functions $h$ are non-smooth and so are not
differentiable whenever the optimal projection $y$ lies on an $(n-k)$-face
for $2 \leq k \leq n$. For example, when $y$ lies on a vertex, changes in $x$
may not have any effect on $y$.
While the gradient obtained from \propref{prop:main} provides a valid local descent
direction (Clarke generalized gradient \cite{ClarkeTAMS:1975}), it can be modified to prefer a zero gradient
at these plateau dimensions by masking the gradient,
and thereby becoming identical to the gradient obtained by numerical 
differentiation methods.

The various constraint functions and gradients are given in
\tabref{tab:lp_projection}. When used to define a projection
operation, only the $L_2$ case has a closed-form solution.
Nonetheless, the gradient of the solution with respect to the input
data for all constraint functions can be calculated using
\propref{prop:main} as summarized in the table.

\begin{table*}[!t]
  \centering
  \newcolumntype{C}{>{\centering\arraybackslash}X}
  \renewcommand*{\arraystretch}{1.3}
  \begin{tabularx}{\textwidth}{@{}l | C C C@{}}
    \multicolumn{1}{l}{} & {$L_2$} & {$L_1$} & {$L_{\infty}$}\\
    \hline
    \rule{0pt}{1.5\normalbaselineskip}
    $\displaystyle h(u)$
    &
    $\displaystyle \| u \|_2 - 1 = \sqrt{\textstyle \sum_{i=1}^{n} u_i^2 } - 1$
    &
    $\displaystyle \| u \|_1 - 1 = {\textstyle \sum_{i=1}^{n}} |u_i| - 1$
    &
    $\displaystyle \| u \|_\infty - 1 = \max_{i} \{|u_i|\} - 1$
    \\[6pt]
    &
    \colorlet{lightgray}{black!25}
    \begin{tikzpicture}
	    \def \x {2.0};
	    \def \y {0.75};
	    \foreach \r in {2, ..., 6}
			\draw[lightgray] (\x, \y) circle (0.05 * \r * \r); 
		\draw[fill=black] (\x, \y) circle (0.05) node[above right] {$x$};
		
		\draw[-{Latex[]}] (0, -1.5) -- (0, 2);
		\draw[-{Latex[]}] (-1.5, 0) -- (3.5, 0);
		
		\draw[thick] (0,0) circle (1);
		\draw[red,dashed] (\x, \y) circle ({sqrt(\x^2 + \y^2) - 1.0});
		\draw[fill=black] ({sqrt(\x^2 / (\x^2 + \y^2)}, {sqrt(\y^2 / (\x^2 + \y^2)}) circle (0.05) node[left] {$y_2$};
		\draw[dashed] (\x, \y) -- ({sqrt(\x^2 / (\x^2 + \y^2)}, {sqrt(\y^2 / (\x^2 + \y^2)});
    \end{tikzpicture}
    &
    \colorlet{lightgray}{black!25}
    \begin{tikzpicture}
	    \def \x {2.0};
	    \def \y {0.75};
	    \foreach \r in {2, ..., 6}
			\draw[lightgray] (\x, \y) circle (0.05 * \r * \r); 
		\draw[fill=black] (\x, \y) circle (0.05) node[above right] {$x$};
		
		\draw[-{Latex[]}] (0, -1.5) -- (0, 2);
		\draw[-{Latex[]}] (-1.5, 0) -- (3.5, 0);
		
		\draw[thick] (0, 1) -- (1, 0) -- (0, -1) -- (-1, 0) -- (0, 1);
		\draw[red,dashed] (\x, \y) circle ({sqrt((\x - 1.0)^2 + \y^2)});
		\draw[fill=black] (1, 0) circle (0.05) node[below=1mm] {$y_1$};
		\draw[dashed] (\x, \y) -- (1, 0);
    \end{tikzpicture}
    &
	\colorlet{lightgray}{black!25}
    \begin{tikzpicture}
	    \def \x {2.0};
	    \def \y {0.75};
	    \foreach \r in {2, ..., 6}
			\draw[lightgray] (\x, \y) circle (0.05 * \r * \r); 
		\draw[fill=black] (\x, \y) circle (0.05) node[above right] {$x$};
		
		\draw[-{Latex[]}] (0, -1.5) -- (0, 2);
		\draw[-{Latex[]}] (-1.5, 0) -- (3.5, 0);
		
		\draw[thick] (-1,-1) rectangle (1, 1);
		\draw[red,dashed] (\x, \y) circle (\x - 1.0);
		\draw[fill=black] (1, \y) circle (0.05) node[below left] {$y_{\infty}$};
		\draw[dashed] (\x, \y) -- (1, \y);
    \end{tikzpicture}
    \\[6pt]
    & closed-form, smooth, & non-smooth ($C^0$), & non-smooth ($C^0$),
    \\
    & unique$^\dagger$ solution & isolated solutions & isolated solutions
    \\[6pt]
    $\displaystyle \dd[Y] h(y)$
    &
    $\displaystyle y$
    &
    $\displaystyle \myvec{\sign{y_i}}^T$
    &
    \begin{tabular}{c}
    $\displaystyle \myvec{\ind{i \in I^\star} \, \sign{y_i}}^T$\\
    $\displaystyle I^\star = \{ i \mid |y_i| \geq |y_j| \; \forall j \}$
	\end{tabular}
    \\[11pt]
    $\displaystyle \dd[YY]^2 h(y)$
    &
    $\displaystyle I - yy^T $
    &
    $\displaystyle \zeros_{n \times n}$
    &
    $\displaystyle \zeros_{n \times n}$
    \\[11pt]
    $\displaystyle \lambda$
    &
    $\displaystyle 1 - \|x\|_2$
    &
    $\displaystyle \sign{y_{i}}(y_{i} - x_i) \; \forall i$
    &
    $\displaystyle \ind{i \in I^\star} \, \sign{y_i} (y_i - x_i)  \; \forall i \in I^\star$
    \\[11pt]
    $\displaystyle \dd y(x)$
    &
    $\displaystyle \frac{1}{\|x\|_2} \left( I - yy^T \right)$
    &
    $\displaystyle I - \frac{\dd[Y] h(y)^T \dd[Y] h(y)}{\dd[Y] h(y) \dd[Y] h(y)^T}$
    &
    $\displaystyle I - \frac{\dd[Y] h(y)^T \dd[Y] h(y)}{\dd[Y] h(y) \dd[Y] h(y)^T}$
    \\[11pt]
    $\displaystyle \dd^{\text{z}} y(x)$
    &
    $\displaystyle \frac{1}{\|x\|_2} \left( I - yy^T \right)$
    &
    $\displaystyle \diag{|\dd[Y] h(y)|} - \frac{\dd[Y] h(y)^T \dd[Y] h(y)}{\dd[Y] h(y) \dd[Y] h(y)^T}$
    &
    $\displaystyle I - \diag{|\dd[Y] h(y)|}$
    \\[-6pt]
    \\
    \hline
  \end{tabularx}
  \caption{
  The gradient of the Euclidean projection onto various $L_p$-spheres
  with constraint functions $h$, when it exists. The projection is found as
  $y(x) = \argmin_{u \in \reals^n} f(x, u)$ subject to $h(u) = 0$, with
  $f(x, u) = \frac{1}{2} \|u - x\|_2^2$. In all cases,
  $B = \dd[X Y]^2 f(x, y) = -I$, and $\dd[YY]^2 f(x, y) = I$.
  The plots show the Euclidean projection of an example point
  $x \in \reals^2$ onto $L_p$-spheres for $p = 1, 2 \text{ and } \infty$.
  Despite not being able to solve the problem in closed form for $L_1$ and
  $L_\infty$ constraints, given a solution the gradient $\dd y(x)$ can still be calculated.
  While $\dd y(x)$ provides a valid local descent direction in all cases,
  the gradient can be altered to prefer a zero gradient along plateau dimensions
  by zeroing the rows and columns corresponding to the zero ($L_1$)
  or non-zero ($L_\infty$) elements of $\dd[Y] h(y)$.
  This gradient $\dd^{\text{z}} y(x)$ is obtained by masking
  $\dd y(x)$ with   $pp^T$ for $L_1$ or $\bar{p} \bar{p}^T$ for $L_\infty$,
  where $p = |\dd[Y] h(y)|$ is treated as a Boolean vector.
  ${}^\dagger$Except for $x = 0$ where the solution is non-isolated (the entire sphere).
  For $L_1$ and $L_\infty$ the solution is unique when $x$ is outside the ball but
  can be non-unique for $x$ inside the ball. 
  }
  \label{tab:lp_projection}
\end{table*}

We can also consider a declarative node that projects onto the unit $L_p$-ball
with output defined as
\begin{align}
\arraycolsep=2pt
\begin{array}{rll}
y_p^\circ \in & \argmin_{u \in \reals^n} & \frac{1}{2} \|u - x\|_2^2 \\
& \text{subject to} & \|u\|_p \leq 1
\end{array}
\end{align}
where we now have an inequality constrained convex optimization problem (for $p \geq 1$).
Here we take the gradient $\dd y_p^\circ$ to be zero if $\|x\|_p < 1$
and $\dd y_p$ otherwise. In words, we set the gradient to zero
if the input already lies inside the unit ball. Otherwise we
use the gradient obtained from projecting onto the $L_p$-sphere.

\section{Implementation Considerations}
\label{sec:implementation}

In this section we provide some practical considerations relating to the
implementation of deep declarative nodes.

\subsection{Vector--Jacobian product}
\label{sec:vec_jac_prod}

As discussed above, the key challenge for declarative nodes is in computing $\dd[] y(x)$
for which we have provided linear algebraic expressions in terms of first- and second-order
derivatives of the objective and constraint functions. In computing the gradient of the
loss function, $J$, for the end-to-end model we compute
\begin{align}
	\dd J(x) &= \dd J(y) \dd y(x)
\end{align}
and the order in which we evaluate this expression can have a dramatic effect on
computational efficiency.

For simplicity consider the unconstrained case (\propref{prop:unconstrained}) and
let $v\transpose = \dd J(y) \in \reals^{1 \times m}$. We have
\begin{align}
	\dd J(x) &= -v\transpose H^{-1} B
	\label{eqn:comp_efficiency}
\end{align}
where $H \in \reals^{m \times m}$ and $B \in \reals^{m \times n}$. The expression
can be evaluated in two distinct ways---either as $(v\transpose H^{-1})B$ or as $v\transpose (H^{-1} B)$
where parentheses have been used to highlight calculation order. Assuming that
$H^{-1}$ has been factored (and contains no special structure), the cost of evaluating
$\dd J(x) \in \reals^{1 \times n}$ is $O(m^2 + mn)$ for the former and $O(m^2n)$ for
the latter, and thus there is a computational advantage to evaluating \eqnref{eqn:comp_efficiency}
from left to right.

Moreover, it is often the case that $n \gg m$ in deep learning models, such as for robust pooling.
Here, rather than pre-computing $B$, which would require $O(mn)$ bytes of precious GPU memory to store,
it is much more space efficient to compute the elements of $\dd J(x)$ iteratively as
\begin{align}
	\dd J(x)_i &= \tilde{v}\transpose b_i
\end{align}
where $\tilde{v} = -H^{-1}v$ is cached and $b_i$ is the $i$-th column of $B$ computed
on the fly and therefore only requiring $O(m)$ bytes.

\subsection{Automatic differentiation}
\label{sec:auto_diff}

The forward processing function for some deep declarative nodes (such as those defined by
convex optimization problems) can be implemented using generic solvers. 
Indeed, \citename{Agrawal:NIPS2019} provide a general framework for specifying
(convex) deep declarative nodes through disciplined convex optimization and providing methods for
both the forward and backward pass.
However, many other interesting nodes will require specialized solvers for their forward
functions (\eg~the coordinate-based SDP solver proposed by \citename{Wang:ICML2019}).
Even so, the gradient calculation in the backward pass can be implemented using generic
automatic differentiation techniques whenever the objective and constraint functions are
twice continuously differentiable (in the neighborhood of the solution). For example, the
following Python code makes use of the \texttt{autograd} package to compute the gradient
of an arbitrary unconstrained deep declarative node with twice differentiable, first-order
objective \texttt{f} at minimum \texttt{y} given input \texttt{x}.

\begin{pycode}
import autograd.numpy as np
from autograd import grad, jacobian

def gradient(f, x, y):
  fY = grad(f, 1)
  fYY = jacobian(fY, 1)
  fXY = jacobian(fY, 0)

  return -1.0 * np.linalg.solve(fYY(x,y), fXY(x,y))
\end{pycode}

If called repeatedly (such as during learning) the partial derivative functions \texttt{fY},
\texttt{fYY} and \texttt{fXY} can be pre-processed and cached. And of course the gradient can instead be
manually coded for special cases and when the programmer chooses to introduce memory and speed efficiencies.
The implementation for equality and inequality constrained problems follows in a straightforward way and
is omitted here for brevity.
We provide full Python and PyTorch reference implementations and examples at \url{http://deepdeclarativenetworks.com}.

\section{Experiments}
\label{sec:experiments}

We conduct experiments on standard image and point cloud classification tasks to
assess the viability of applying declarative networks to challenging computer vision
problems. Our goal is not to obtain state-of-the-art results. Rather we aim to validate
the theory presented above and demonstrate how declarative nodes can easily be
integrated into non-trivial deep learning pipelines.
To this end, we implement the pooling and projection operations from \secref{sec:examples}.
For efficiency, we use the symbolic derivatives obtained in Tables~\ref{tab:robust_pooling} and \ref{tab:lp_projection}, and never compute the Jacobian directly, instead computing the vector--Jacobian product to reduce the memory overhead.
All code is available at \url{http://deepdeclarativenetworks.com}.

For the point cloud classification experiments with robust pooling, we use the ModelNet40 CAD dataset~\cite{Wu:CVPR15}, with 2048 points sampled per object, normalized into a unit ball~\cite{Qi:CVPR2017}.
Each model is trained using stochastic gradient descent for 60 epochs with an initial learning rate of 0.01, decaying by a factor of 2 every 20 epochs, momentum (0.9), and a batch size of 24. All models, variations of PointNet~\cite{Qi:CVPR2017} implemented in PyTorch, have 0.8M parameters.
Importantly, pooling layers do not add additional parameters.

The results are shown in Tables~\ref{tab:modelnet_pooling} and \ref{tab:modelnet_pooling_2}, for a varying fraction of outliers seen during training and testing (\tabref{tab:modelnet_pooling}) or only during testing, that is trained without outliers (\tabref{tab:modelnet_pooling_2}).
We report top-1 accuracy and mean Average Precision (mAP) on the test set.
Outlier points are sampled uniformly from the unit ball, randomly replacing existing inlier points, following the same protocol as \citename{Qi:CVPR2017}.
The robust pooling layer replaces the max pooling layer in the model, with the quadratic penalty function being denoted by Q, pseudo-Huber by PH, Huber by H, Welsch by W, and truncated quadratic by TQ.
The optimizer for the last two (non-convex) functions is initialized to both the mean and the median, and then the lowest local minimum selected as the solution. It is very likely that RANSAC~\cite{Fischler:1981} search would produce better results, but this was not tested.
For all experiments, the robustness parameter $\alpha$ (loosely, the inlier threshold) is set to one.

We observe that the robust pooling layers perform significantly better than max pooling and quadratic (average) pooling when outliers are present, especially when not trained with the same outlier rate.
There is a clear trend that, as the outlier rate increases, the most accurate model tends to be more robust model (further towards the right).

For the image classification experiments with $L_p$-sphere and $L_p$-ball projection, we use the ImageNet 2012 dataset~\cite{ImageNet:2009}, with the standard single central $224\times224$ crop protocol. Each model is trained using stochastic gradient descent for 90 epochs with an initial learning rate of 0.1, decaying by a factor of 10 every 30 epochs, weight decay (1e-4), momentum (0.9), and a batch size of 256. All models, variations of ResNet-18~\cite{He:CVPR2016} implemented in PyTorch, have 11.7M parameters. As with the robust pooling layers, projection layers do not add additional parameters.

The results are shown in \tabref{tab:imagenet_projection}, with top-1 accuracy, top-5 accuracy, and mean Average Precision (mAP) reported on the validation set.
The projection layer, prepended by a batch normalization layer with no learnable parameters, is inserted before the final fully-connected output layer of the network. The features are pre-scaled according to the formula $\frac{2}{3} \median_i \|f_i\|_p$ where $f_i$ are the batch-normalized training set features of the penultimate layer of the pre-trained ResNet model. This corresponds to scaling factors of 250, 15, and 3 for $p=1$, $2$, and $\infty$ respectively, and can be thought of as varying the radius of the $L_p$-ball. The chosen scale ensures that the projection affects most features, but is not too aggressive.

The results indicate that feature projection improves the mAP significantly, with a more modest increase in top-1 and top-5 accuracy. This suggests that projection encourages a more appropriate level of confidence about the predictions, improving the calibration of the model.

\begin{table}[!t]\centering
	\caption{
		The effect of robust pooling layers on point cloud classification results for the ModelNet40 dataset \cite{Wu:CVPR15}, with varying percentages of outliers (O) and the same rate of outliers seen during training and testing.
		Outliers points are uniformly sampled from the unit ball.
		PointNet \cite{Qi:CVPR2017} is compared to our variants that replace max pooling with robust pooling: quadratic (Q), pseudo-Huber (PH), Huber (H), Welsch (W), and truncated quadratic (TQ), all trained from scratch.
		Top-1 accuracy and mean average precision are reported.
	}
	\label{tab:modelnet_pooling}
	\newcolumntype{C}{>{\centering\arraybackslash}X}
	\renewcommand*{\arraystretch}{1.3}
	\setlength{\tabcolsep}{1pt}
	\begin{tabularx}{\columnwidth}{@{}r | C C C C C C | C C C C C C@{}}\hline
	O & \multicolumn{6}{c |}{Top-1 Accuracy \%} & \multicolumn{6}{c}{Mean Average Precision $\times 100$}\\
	\% & \cite{Qi:CVPR2017} & Q & PH & H & W & TQ & \cite{Qi:CVPR2017} & Q & PH & H & W & TQ\\\hline
0 & \textbf{88.4} & 84.7 & 84.7 & 86.3 & 86.1 & 85.4 & \textbf{95.6} & 93.8 & 95.0 & 95.4 & 95.0 & 93.8\\
10 & 79.4 & 84.3 & 85.6 & 85.5 & \textbf{86.6} & 85.5 & 89.4 & 94.3 & 94.6 & \textbf{95.1} & 94.6 & 94.7\\
20 & 76.2 & 84.8 & 84.8 & 85.2 & \textbf{86.3} & 85.5 & 87.8 & 94.8 & 95.0 & \textbf{95.0} & 94.8 & 95.0\\
50 & 72.0 & 84.0 & 83.1 & 83.9 & \textbf{84.3} & 83.9 & 83.3 & 93.8 & 93.5 & 94.3 & 94.8 & \textbf{94.8}\\
90 & 29.7 & 61.7 & 63.4 & 63.1 & \textbf{65.3} & 61.8 & 38.9 & 76.8 & 78.7 & 78.5 & \textbf{79.1} & 76.6\\
\hline
	\end{tabularx}
\end{table}

\begin{table}[!t]\centering
	\caption{
		The effect of robust pooling layers on point cloud classification results for the ModelNet40 dataset \cite{Wu:CVPR15}, with varying percentages of outliers (O) and no outliers seen during training.
		During testing, outlier points are uniformly sampled from the unit ball.
		Models are identical to those in \tabref{tab:modelnet_pooling}.
		Top-1 accuracy and mean average precision are reported.
	}
	\label{tab:modelnet_pooling_2}
	\newcolumntype{C}{>{\centering\arraybackslash}X}
	\renewcommand*{\arraystretch}{1.3}
	\setlength{\tabcolsep}{1pt}
	\begin{tabularx}{\columnwidth}{@{}r | C C C C C C | C C C C C C@{}}\hline
	O & \multicolumn{6}{c |}{Top-1 Accuracy \%} & \multicolumn{6}{c}{Mean Average Precision $\times 100$}\\
	\% & \cite{Qi:CVPR2017} & Q & PH & H & W & TQ & \cite{Qi:CVPR2017} & Q & PH & H & W & TQ\\\hline
0 & \textbf{88.4} & 84.7 & 84.7 & 86.3 & 86.1 & 85.4 & \textbf{95.6} & 93.8 & 95.0 & 95.4 & 95.0 & 93.8\\
1 & 32.6 & 84.9 & 84.7 & \textbf{86.4} & 86.2 & 85.3 & 48.6 & 93.8 & 95.1 & \textbf{95.3} & 95.1 & 93.0\\
10 & 6.47 & 83.9 & 84.6 & 85.3 & \textbf{86.0} & 85.9 & 8.20 & 93.4 & 94.8 & 94.4 & \textbf{94.9} & 93.9\\
20 & 5.95 & 79.6 & 82.8 & 81.1 & 84.7 & \textbf{84.9} & 7.73 & 91.9 & 93.4 & 92.7 & 94.2 & \textbf{94.6}\\
30 & 5.55 & 70.9 & 74.2 & 72.2 & 77.6 & \textbf{83.2} & 6.00 & 87.8 & 89.5 & 85.1 & 90.9 & \textbf{92.8}\\
40 & 5.35 & 55.3 & 59.1 & 55.4 & 63.1 & \textbf{75.6} & 6.41 & 77.6 & 80.2 & 72.7 & 83.2 & \textbf{90.6}\\
50 & 4.86 & 32.9 & 36.0 & 34.6 & 44.1 & \textbf{57.9} & 5.68 & 62.3 & 60.2 & 60.1 & 66.4 & \textbf{85.3}\\
60 & 4.42 & 14.5 & 16.2 & 18.1 & 27.1 & \textbf{30.6} & 5.08 & 39.1 & 36.3 & 38.5 & 42.7 & \textbf{68.5}\\
70 & 4.25 & 5.03 & 6.33 & 7.95 & \textbf{14.1} & 11.9 & 4.66 & 22.5 & 19.3 & 18.4 & 25.7 & \textbf{47.9}\\
80 & 3.11 & 4.10 & 4.51 & 5.64 & \textbf{8.88} & 5.11 & 4.21 & 10.8 & 8.91 & 8.98 & 14.9 & \textbf{26.7}\\
90 & 3.72 & 4.06 & 4.06 & 4.30 & \textbf{5.68} & 4.22 & 4.49 & 8.20 & 5.98 & 5.80 & 8.37 & \textbf{9.78}\\
\hline
	\end{tabularx}
\end{table}

\begin{table}[!t]\centering
	\caption{The effect of projection layers on image classification results for the ImageNet 2012 dataset \cite{ImageNet:2009}. Top-1 and top-5 accuracy, and mean Average Precision (mAP) are reported.
	All models are trained from scratch, with the exception of ResNet-18-pt, which uses the PyTorch pre-trained weights.
	}
	\label{tab:imagenet_projection}
	\newcolumntype{C}{>{\centering\arraybackslash}X}
	\renewcommand*{\arraystretch}{1.3}
	\setlength{\tabcolsep}{2pt}
	\begin{tabularx}{\columnwidth}{@{}l C C C@{}}\hline
		Model & Top-1 Acc. \% & Top-5 Acc. \% & mAP $\times 100$\\\hline
		ResNet-18 & 69.80 & 89.26 & 58.97\\
		ResNet-18-pt & 69.76 & 89.08 & 53.76\\
		ResNet-18-$L_1$Sphere & 69.92 & 89.38 & 62.72\\
		ResNet-18-$L_2$Sphere & \textbf{70.66} & 89.60 & 71.97\\
		ResNet-18-$L_\infty$Sphere & 70.03 & 89.22 & 63.98\\
		ResNet-18-$L_1$Ball & 70.17 & 89.47 & 61.26 \\
		ResNet-18-$L_2$Ball & 70.59 & \textbf{89.70} & \textbf{72.43}\\
		ResNet-18-$L_\infty$Ball & 70.06 & 89.29 & 63.33\\\hline
	\end{tabularx}
\end{table}

\section{Conclusion}
\label{sec:conclusion}

In the preceding sections we have presented some theory and practice for deep declarative
networks. On the one hand we developed nothing new---argmin can simply be viewed as yet another function and
all that is needed is to work out the technical details of how to compute its derivative.
On the other hand deep declarative networks offers a new way of thinking about network design
with expressive processing functions and constraints but without having to worry about
implementation details (or having to back-propagate gradients through complicated algorithmic
procedures).

By facilitating the inclusion of declarative nodes into end-to-end learnable models,
network capacity can be directed towards identifying features and patterns in the data
rather than having to re-learn an approximation for theory that is already well-established
(for example, physical system models) or enforce constraints that we know must hold (for example, that the
output must lie on a manifold). As such, with deep declarative networks, we have the potential
to create more robust models that can generalize better from less training data (and with
fewer parameters) or, indeed, approach problems that could not be tackled by deep learning
previously as demonstrated in recent works such as \citename{Amos:ICML2017} and \citename{Wang:ICML2019}.

As with any new approach there are some shortcomings and much still to do. For example,
it is possible for problems with parametrized constraints to become infeasible during learning, so care should be taken to avoid parameterizations that may result in an empty feasible set. 
More generally, optimization problems can be expensive to solve and those with multiple solutions can create difficulties for end-to-end learning. Also more theory needs to be developed around non-smooth objective functions.
We have presented some techniques for dealing with problems with inequality constraints
(when we may only have a single-sided gradient) or when the KKT optimality conditions cannot be guaranteed
but there are many other avenues to explore. Along these lines it would also be interesting
to consider the trade-off in finding general descent directions rather than exact gradients
and efficient approximations to the forward processing function, especially in the early
stages of learning.

The work of \citename{Chen:NIPS2018} that shows how to differentiate through an ordinary
differential equation also presents interesting directions for future work. Viewing
such models as declarative nodes with differential constraints suggests many
extensions including coupling ordinary differential equations with constrained optimization problems, which may be useful for physical simulations or modeling dynamical systems.

Finally, the fact that deep declarative nodes provide some guarantees on their output
suggests that a principled approach to analyzing the end-to-end behavior of a deep declarative
network might be possible. This would be especially useful in deploying deep learned
models in safety critical applications such as autonomous driving or complex control systems.
And while there are many challenges yet to overcome, the hope is that deep
declarative networks will deliver solutions to problems faced by existing deep neural networks
leading to better robustness and interpretability.


\ifCLASSOPTIONcompsoc
  \section*{Acknowledgments}
\else
  \section*{Acknowledgment}
\fi

We give warm thanks to Bob Williamson and John Lloyd for helpful discussions, and
anonymous reviewers for insightful suggestions.
We thank Itzik Ben-Shabat for recommending experiments on point cloud classification.

\ifCLASSOPTIONcaptionsoff
  \newpage
\fi


{

}

\end{document}